\documentclass[english,a4paper,10pt]{article}

\usepackage{graphicx}
\usepackage{latexsym}
\usepackage[algo2e,ruled,vlined]{algorithm2e}
\usepackage{amsmath}
\usepackage{amsthm}
\usepackage{amssymb}
\usepackage{graphicx}
\usepackage{thmtools}
\usepackage{hyperref}
\usepackage{varioref}

\newtheorem{theorem}{Theorem}
\newtheorem{definition}[theorem]{Definition}
\newtheorem{remark}[theorem]{Remark}
\newtheorem{lemma}[theorem]{Lemma}
\theoremstyle{remark}
\newtheorem*{claim}{Claim}

\DeclareMathOperator{\IsFitter}{\textsc{IsOrdered}}
\DeclareMathOperator{\Swap}{\textsc{Swap}}
\newcommand{\N}{\mathbb{N}}
\newcommand{\R}{\mathbb{R}}

\newcommand*\rfrac[2]{{}^{#1}\!\!/\!{}_{#2}}
\newcommand{\eps}{\varepsilon}
\newcommand{\MO}{\mathcal{O}}
\newcommand{\E}{\mathbb{E}}
\newcommand{\MP}{\mathbb{P}}
\newcommand{\Tcest}{T_{\textsf{est}}}
\newcommand{\Tconv}{T_{\textsf{conv}}}
\newcommand{\Swapsort}{\textsc{SwapSort}}

\begin{document}

\title{Sorting by Swaps with Noisy Comparisons\footnote{
An extended abstract of this paper has been presented at Genetic and Evolutionary Computation Conference (GECCO 2017).}}

\author{Tom{\'{a}}{\v{s}} Gaven{\v{c}}iak\textsuperscript{1,3} \and Barbara Geissmann\textsuperscript{2,4} \and Johannes Lengler\textsuperscript{2}}

\date{\textsuperscript{1} Charles University in Prague\\ \textsuperscript{2} ETH Zurich}

\maketitle

\footnotetext[3]{Research supported by the Czech Science Foundation (GA\v{C}R) project 17-10090Y ``Network optimization''.}
\footnotetext[4]{Research supported by the Swiss National Science Foundation (SNSF), project number 200021\_165524.}

\begingroup\renewcommand\thefootnote{}
\footnotetext{\tt gavento@kam.mff.cuni.cz, \{barbara.geissmann, johannes.lengler\}@inf.ethz.ch}
\endgroup

\begin{abstract}
We study sorting of permutations by random swaps if each comparison gives the wrong result with some fixed probability $p<1/2$. We use this process as prototype for the behaviour of randomized, comparison-based optimization heuristics in the presence of noisy comparisons. As quality measure, we compute the expected fitness of the stationary distribution. To measure the runtime, we compute the minimal number of steps after which the average fitness approximates the expected fitness of the stationary distribution.

We study the process where in each round a random pair of elements at distance at most $r$ are compared. We give theoretical results for the extreme cases $r=1$ and $r=n$, and experimental results for the intermediate cases. We find a trade-off between faster convergence (for large $r$) and better quality of the solution after convergence (for small~$r$).

\end{abstract}

\section{Introduction}
\label{sec:introduction}

Randomized optimization heuristics like evolutionary algorithms (EAs) have become important practical tools for optimization problems that are too complex to solve exactly~\cite{dasgupta2013evolutionary}, and some efforts have been made to understand such search heuristics theoretically~\cite{auger2011theory,jansen2013analyzing,neumann2013bioinspired}. Mostly, the optimization problem is given by some unknown \emph{fitness function} $f: \mathcal{S} \to \R$ to be minimized (or maximized), where the \emph{search space} $\mathcal{S}$ is the set of all possible solutions. 

An important and classical aspect of EAs is how robust their performance is in the presence of noise~\cite{droste2004analysis,beyer2000evolutionary}. This theme has gained increased attention in the last few years, \cite{friedrich2016compact,merelo2016statistical,cauwet2016algorithm,astete2016simple,astete2016analysis,astete2015evolution,akimoto2015analysis,liu2014mathematically,qian2015analyzing,qian2016effectiveness,GiessenK16}, see~\cite{rakshit2016noisy} for a comprehensive review. 
Mostly, noise is modeled by imperfect fitness function evaluations that -- instead of the exact fitness value -- return a perturbed value (e.g., by a Gaussian additive term). 
This model is very accurate for algorithms which explicitly use the fitness function. In particular, in the setting of \emph{black-box optimization}, the function $f$ can be accessed (exclusively) by evaluating $f(s)$ for any search point $s \in \mathcal{S}$ that the algorithm may choose. However, it may be less useful for algorithms which do not fall into this category. In particular, a broad class of EAs are \emph{comparison-based}, i.e., they do not evaluate $f(s)$, but rather they only compare which one of two given search points $s$ and $s'$ is the better one. For example, if genetic algorithms are used to optimize chess engines, then the selection process will not rely on fitness values, but rather on comparisons (e.g., by tournaments) between different engines. 

Another interesting example is sorting by swaps, the topic of this paper: In order to compare two permutations $s$ and $s'$ that differ by a single swap, it is sufficient to consider the two elements that were swapped, and their positions. While there are global fitness functions that measure the unsortedness of a permutation (cf.\! Definition~\ref{def:fitnessfunctions} below), the information that comparison-based EAs consider in each round does not generally suffice to predict how much this unsortedness will decrease, which makes it inappropriate to model the noise by a perturbed fitness function.
Therefore, we rather assume that the process of comparing itself is error-prone. More precisely, we assume that every comparison gives a false output with some probability $p<1/2$. This algorithms falls into the class of \emph{noisy comparison-based $(1+1)$ EAs}, as described by Algorithm~\ref{alg:1+1}. The general approach can be extended to the more general class of $(\mu+\lambda)$ algorithms, which keep a whole population in the memory rather than a single search point. However, there are some additional subtleties which we want to avoid here -- for example, which of the $\mu$ individuals of the last population the algorithm actually chooses as output.


There are some problems that make it more complicated to define a theoretical evaluation of comparison-based EAs in the presence of noise. For example, the standard measure for the runtime of an EA in theoretical studies is the number of fitness evaluations until an optimal solution is hit for the first time. This is arguably unsuitable for noisy comparison-based algorithms, because even if they do find an optimum, due to the noisy measurements, the algorithm might not be able to recognise it as a best-so-far solution. Moreover, in a noisy environment the global optimum may not have a practical advantage over some other search points. Thus if we expect the algorithm to find the global optimum, we force it to spend possibly a lot of time on searching through the set of all solutions, which are practically indistinguishable. An alternative which avoids the aforementioned problems is the fixed-budget approach, in which we ask for the best solution that the algorithm obtains within a fixed budget~$B$ of function evaluations. However, this approach has the disadvantage that it needs an additional parameter: the budget~$B$. Instead, we aim for parameter-free alternatives. 



As a solution, we regard the algorithm as a Markov chain over the search space $\mathcal{S}$~\cite{droste2004analysis}. As such, the algorithm converges to some \emph{stationary distribution} of states, and the fitness of the output can naturally be defined as the expected fitness of the stationary distribution. Moreover, there is a well-established notion of mixing time, which is a natural measure for the time after which the solution does not change further. However, this notion is not necessarily a good measure for the time until good solutions are found. The mixing time measures the time until the \emph{genotypical distribution} of solutions becomes stable, i.e., until the distribution on $\mathcal S$ has converged. But if there is, for example, a large plateau of almost-optimal solutions, then the time to reach this plateau may be significantly shorter than the mixing time. Thus we rather measure the runtime of the algorithm by the minimal time until the expected fitness of the solution is close to the expected fitness of a solution of the stationary distribution. 
This issue, that for algorithms in noisy environments one should not analyze  the hitting time of the optimum, but rather of some neighbourhood, has also been studied in \cite{DangJL17,KotzingLW15}. 
In EA jargon, this resembles (though not exactly) convergence of the \emph{phenotypical distribution} instead of the \emph{genotypical distribution}.


As mentioned before, we study the notions discussed above for the \emph{sorting} problem, i.e., the individuals are permutations of the set $\{1,\ldots,n\}$. This problem has been introduced in the seminal paper of Scharnow, Tinnefeld and Wegener~\cite{scharnow2004analysis}, and it has been studied in different encodings~\cite{doerr2008directed}. Several mutation operators are discussed in~\cite{scharnow2004analysis}, one of which is the \emph{swap operator} $\Swap(s,i,j)$, which exchanges the elements at positions $i$ and $j$ in $s$. In this paper we will only consider $\Swap$ operations since they are the only ones for which the algorithm can decide in constant time whether the operation is advantageous or not (i.e., if the two elements were in the wrong order before swapping or not). Thus $\Swap$ operations allow naturally to compare parent and offspring without explicitly accessing the fitness of either search point. For $0<p<1/2$ and $1\leq r \leq n$ we will study the algorithm $\Swapsort_{p,r}$ as given by Algorithm~\ref{alg:1+1}, which produces a random offspring by choosing uniformly at random two elements in distance at most $r$ and swapping them. The comparison between offspring and parent yields the correct answer with probability $1-p$, and the wrong answer with probability $p$. While the extreme cases $r=1$ (adjacent swaps)~\cite{barbaraPaper,benjamini2005} and $r=n$ (arbitrary swaps)~\cite{barbaraPaper,scharnow2004analysis} have been studied before in related settings, this is the first time that general $r$ is considered. However, we will see that both extreme cases have their drawbacks in terms of speed of convergence (for $r=1$) and quality of the final solution (for $r=n$). A similar trade-off has also been observed in~\cite{barbaraPaper} in a slightly different noise model. Thus we also include general values of $r$ in our analysis to interpolate between both cases. Note that the parameter $r$ determines the size of the neighborhood of each permutation, which is $\Theta(rn)$ and ranges from $\Theta(n)$ for $r=1$ to $\Theta(n^2)$ for $r=n$.

\begin{algorithm2e}[b]
	\textbf{Initialize:}
	Sample $\pi^{(0)}$ uniformly at random from $\mathcal{S}_n$.
	
	\textbf{Optimize:}
	\For{$t=1,2,3,\ldots$}{
		Choose $i,j$ uniformly at random with $1 \leq j-i \leq r$.
				
		\eIf{$\IsFitter_p(\pi^{(t-1)}_i,\pi^{(t-1)}_j)$}{$\pi^{(t)} \gets \pi^{(t-1)}$.}{$\pi^{(t)} \gets \Swap(\pi^{(t-1)},i,j)$.}
	}
	\caption{The noisy algorithm $\Swapsort_{p,r}$ maintains a permutation $\pi \in \mathcal S_n$ of $\{1,\ldots,n\}$. In each round it chooses two random indices $i < j$ of distance at most $r$, and compares $\pi_i$ with $\pi_j$. The comparison is noisy, so with probability $p$ it gives the errroneous result. If the algorithm believes that the elements $\pi_i$ and $\pi_j$ are in correct order, then it does nothing; otherwise it swaps them.}
	\label{alg:1+1}
\end{algorithm2e}

\begin{algorithm2e}[b]
	\textbf{$\IsFitter_p(a,b$)}
	
	With probability $p$, set $error \gets$ True; otherwise set $error \gets$ False.
	
	\eIf{$a<b$}{$ordered \gets$ True.}{$ordered \gets$ False.}
	
	\textbf{return} $(ordered \oplus error)$.
	\caption{The operator $\IsFitter_p(a,b)$ checks whether $a<b$, but makes an error with probability $p$.}
	\label{alg:isFitter}
\end{algorithm2e}

\smallskip
To measure the quality of the solution, we need to assume a \emph{ground truth}, i.e., an unknown underlying fitness function. We consider the following options for a permutation $\pi$ of the set $\{1,\dots, n\}$, see Section~\ref{sec:formal} for formal definitions. 
\begin{itemize}
	\item $D(\pi)$ is the \emph{total dislocation}, i.e., the sum of all distances of elements $i$ from their positions $\pi(i)$, also known as \emph{Spearman's footrule}~\cite{diaconis1977}.
	\item $I(\pi)$ is the \emph{number of inversions}, i.e., the number of pairs $(i,j)$, $i<j$ such that $\pi(i) > \pi(j)$~\cite{scharnow2004analysis}.
	\item {$W(\pi)$ is the \emph{weighted number of inversions}, where each inversion $(i,j)$ is weighted by $j-i$}~\cite{barbaraPaper}.
\end{itemize}

\subsection{Relations to other areas}\label{sec:relations}

The process of sorting by random comparisons and swaps and its variants have been studied in other contexts as well.

Similar Markov chains are studied as \emph{(biased) card shuffling processes} with a focus on the mixing time as a measure of efficiency of the shuffling~\cite{benjamini2005}. In biased card shuffling, cards are swapped to be in the sorted order with probability $1-p$ and anti-sorted otherwise.

The \mbox{$0$-$1$~sequence} sorting considered in Lemma~\ref{lem:zero-one-quality} is related to the \emph{asymmetric exclusion process}, studied in statistical physics to model the dynamics of continuous particle diffusion in an infinite one-dimensional space~\cite{SPITZER1970,Tracy2009}. We can model the process $\Swapsort_{p,1}$ as a cano\-nical ensemble using $I$ as the energy function\footnote{Using $W$ or $D$ as the energy leads to different, less local error probability functions.} and temperature $1/\log({1}/{p}-1)$. We omit the details in this paper.

Any particular sequence of comparison operations performed by the $\Swapsort$ process gives rise to a \emph{comparator network}\footnote{Generally not a \emph{sorting network} since the resulting networks are generally not sorting every input.} since the selection of pairs to compare is independent of the actual values. Therefore our results directly translate to random comparator networks generated by appending comparators on uniformly chosen wires in distance at most $r$. We omit the straightforward reformulation of the results. Sorting networks with unreliable comparators have been studied by, e.g.,~\cite{assaf1991fault,ma1994fault}.

The results of Giesen et al.~\cite{Giesen2009} imply that for every $\epsilon>0$, every pre-sorting algorithm (and therefore also every comparator network) that achieves a permutation $\pi$ with  $\E(D(\pi))=\MO(n^{2-\epsilon})$ (or $\E(I(\pi))$, equivalently) in expectation requires $\Omega(n\log n)$ comparisons. This also implies a lower bound of $\Omega(n\log n)$ on the number of steps of any random noisy almost-sorting process to stabilize at $\E(I(\pi))=\MO(n^{2-\epsilon})$. Experimentally, it seems this is the case for $r=n^{1-\epsilon}$. However, the experiments suggest that for such $r$, the runtime is actually closer to $n^2/r=n^{1+\epsilon}\gg n\log n$.

To the best of our knowledge, the results of this paper have not been known in the above contexts unless explicitly referenced.

\subsection{Algorithm}\label{sec:algorithm}
Let $\mathcal{S}_n$ be the set of all permutations of $\{1,\ldots,n\}$, i.e., the set of all bijective functions $\pi:\{1,\ldots,n\} \to \{1,\ldots,n\}$. For convention,
let $\pi(i)$ denote the position of element $i$ in $\pi$ and let $\pi_i$ be the element at position $i$.  For $\pi \in \mathcal{S}_n$ and $1 \leq i, j \leq n$, $i \neq j$, we define the operator $\Swap(\pi,i,j)$ to be the permutation $\pi'$ given by $\pi'_j=\pi_i$ and $\pi'_i=\pi_j$, and $\pi'_k=\pi_k$ for all other indices $i\neq k\neq j$. For a parameter $r \in \{1,\ldots,n\}$ and a parameter $0<p <1/2$, we define $\Swapsort_{p,r}$ to be the $(1+1)$ evolutionary algorithm given by Algorithm~\ref{alg:1+1}, where the $\IsFitter_p$-operator is given by Algorithm~\ref{alg:isFitter}.


We say that a particular swap $\Swap(\pi,i,j)$ is \emph{good} if $\pi_i$ and $\pi_j$ are correctly sorted after swapping, i.e., if $i < j$ and $\pi_i > \pi_j$, or if $i>j$ and $\pi_i < \pi_j$, before the swap. 
Similarly, we say that a swap $\Swap(\pi,i,j)$ is \emph{bad} if $\pi_i$ and $\pi_j$ are wrongly sorted after swapping.

\medskip
Hereafter, we will regard this sorting algorithm as a Markov chain (consider Section~\ref{sec:markov} for definitions and a short introduction to Markov chains). 
We use $S_{p,r}$ to denote the Markov chain associated with the algorithm $\Swapsort_{p,r}$.
Each iteration in Algorithm~\ref{alg:1+1} corresponds to one transition step of the Markov chain~\cite{nix1992modeling}.

We use $ S_{p,r}(\pi) $ to denote the process $S_{p,r}$ starting from permutation $\pi$, 
and we use $S_{p,r}^{\ t}(\pi)$ to denote its probability distribution after $t$ transition steps. If $ S_{p,r}(\pi)$ is clear from the context, then we use $\pi^{(t)}$ to denote a random permutation after $t$ steps.
Furthermore, we write $S_{p,r,n}$ to denote the process starting from a uniformly random $n$-element permutation, and $S_{p,r,n}^{\;\infty}$ to denote the unique stationary distribution (which exists and to which $S_{p,r}$ converges, since $S_{p,r}$ is irreducible and aperiodic as discussed in Section~\ref{sec:general-properties}).

\smallskip
Note that the definition of $S_{p,r}$ as well as $I$ and $W$ (in Remark~\ref{rem:sequences}) can be extended to all $n$-element sequences (even with repeated elements) since for comparison-based sorting, the actual multi-set of values does not matter and does not change in the process. For the number of inversions, the properties of the Markov chain and the upper bounds on the result quality would also carry over since we can perturb the values of the identical elements, which only increases $I$. We leave the details to the interested reader.

\subsection{Our Results}
We study the problem of sorting the set $\{1,\ldots,n\}$ and we consider the Markov chain $S_{p,r}$ over the state space $\mathcal{S}_n$ for the two extreme cases $r=1$ (adjacent swaps only) and $r=n$ (arbitrary swaps), as well as for the general case $1\le r\le n$ (bounded-range swaps).
We now provide an overview of the most important theorems, the proofs however will follow in Sections~\ref{sec:adj} to \ref{sec:bounded}.

For a probability distribution $q$ over all permutations in $\mathcal{S}_n$, we use
$\E(I(q))$ to denote the expected number of inversions of a random permutation $\pi$ chosen with probability $q(\pi)$, i.e.,
\[
\E(I(q)) := \sum_{\pi\in\mathcal{S}_n} q(\pi)\cdot I(\pi)\, .
\]
Similarly, we use $\E(W(q))$ and $\E(D(q))$ for the expected weighted number of inversions and the expected total dislocation, respectively.
 
\subsubsection{Adjacent swaps}

We first consider $S_{p,1}$, which only swaps adjacent elements. Our first result shows that this process has high fitness in the stationary distribution. In particular, even the weighted number of inversions is only of linear order, i.e., on average each element is contained in at most a constant number of inversions, and the worst inversion of each element bridges on average only a constant distance.

\begin{restatable}[Adjacent Swaps]{theorem}{TheoremOne}\label{thm:adj-quality}
	Let $p_{\max}<{1}/{3}$ be constant. Then for any $0<p <p_{\max}$ it holds
	\[p\cdot(n-1)\leq\E(I(S^{\;\infty}_{p,1}))\leq\E(W(S^{\;\infty}_{p,1}))\leq n\cdot\frac{2p}{(1-3p)}+2^{-\Omega(n)}\ .\]
	Moreover, for a random permutation $\pi$ from  $S^{\;\infty}_{p,1}$, with high probability\footnote{That is with probability at least $1-1/n$.}, $$I(\pi)\leq W(\pi)=\MO(n\log n)\, ,$$ and the maximum dislocation is $$\max_{1\leq i \leq n}|i - \pi(i)| = \MO(\log n)\, .$$ 
\end{restatable}

Note that for a fixed $p_{\max}$ the bounds on the expectations are asymptotically tight. For sufficiently small $p$ and large $n$, the ratio of the upper and lower bounds is close to $2$ for both $I$ and $W$. This is illustrated in Figure~\ref{fig:IW-by-p-r1} together with experimental results. Experiments of Section~\ref{sec:conv-quality} suggest that $\E(I(S^{\;\infty}_{p,1}))\simeq f_1(p)\cdot n$ and $\E(W(S^{\;\infty}_{p,1}))\simeq f_2(p)\cdot n$ for some (unknown) functions $f_1$ and $f_2$.
\medskip

Our next result shows that the good fitness in the stationary distribution for $r=1$ comes at the cost of a large convergence time of $\Theta(n^2)$ for both $I$ and $W$. Let 
\begin{align*}
\Tconv^{I}(\epsilon) &:= \min\left\{t\in\N \; : 
\;  \bigg\lvert\frac{\max_{\pi\in\mathcal{S}_n}\{\E(I(S^{\ t}_{p,1}(\pi)))\}}{\E(I(S^{\;\infty}_{p,1}))}-1\bigg\rvert<\epsilon\right\} \, ,\\ 
\Tconv^{W}(\epsilon) &:= \min\left\{t\in\N \; : 
\;  \bigg\lvert\frac{\max_{\pi\in\mathcal{S}_n}\{\E(W(S^{\ t}_{p,1}(\pi)))\}}{\E(W(S^{\;\infty}_{p,1}))}-1\bigg\rvert<\epsilon\right\}
\end{align*}
be the times until $S_{p,1}$ has approached the quality of its stationary distribution up to an error of $\eps$. 

\begin{restatable}[Convergence Time]{theorem}{TheoremTwo}\label{thm:adj-speed}
	For any constant $0<p<\rfrac{1}{2}$ and any constant error $\varepsilon>0$,
	\begin{align*}
	\Tconv^{I}(\epsilon)=\Theta(n^2) && \text{and} && \Tconv^{W}(\epsilon)=\Theta(n^2)\, .
	\end{align*}
\end{restatable}

We also run experiments on the convergence times in Section~\ref{sec:conv-time}. For $p\leq 0.2$, the measured convergence times are between $n^2$ and $2n^2$ (within $95\,\%$ confidence).

\subsubsection{Arbitrary swaps}

Now we turn to $S_{p,n}$, which may swap any pair of elements. We do not provide theoretical results on the convergence time in this case but refer to future work, since the analysis is more complicated as this Markov chain is not reversible. Experiments of Section~\ref{sec:conv-time}, however, suggest that the convergence time is almost $n$ times faster than for $r=1$, and in particular suggest convergence time to be bounded by $\MO(n\log n)$. This would also be consistent with the \mbox{mixing} time of a random card-swapping process shown to be $\MO(n\log n)$ by Diaconis \cite[chapter 3D]{opac-b1087294}. However, this increase in speed comes at a high cost, since the quality of the solution is dramatically worse than for $S_{p,1}$.

\begin{restatable}[Arbitrary Swaps]{theorem}{TheoremThree}\label{thm:any-exp-I-W}
	For any $0<p<{1}/{2}$, 
	\begin{align*}
	&&\Omega(pn^2) =\frac{p(n^2-1)}{6} \leq \E(D(S^{\;\infty}_{p,n}))\leq 2\cdot\E(I(S^{\;\infty}_{p,n}))= \MO(p^{1/3}n^2)\, \\ \text{and} && \Omega(pn^3)=\frac{pn^3}{648} \leq \E(W(S^{\;\infty}_{p,n}))=\MO(p^{1/2}n^3) \, .
	\end{align*}
	In particular, if $p$ is a constant
	\begin{align*}
	\E(D(S^{\;\infty}_{p,n})) = \Theta(n^2) ,&&  \E(I(S^{\;\infty}_{p,n})) = \Theta(n^2) , && \E(W(S^{\;\infty}_{p,n})) = \Theta(n^3)\,  .
	\end{align*}
\end{restatable}

For a random permutation $\pi\in\mathcal{S}_n$, it holds that $\E(I(\pi)) =\frac{1}{2}{n \choose 2}$, $\E(D(\pi)) =\frac{n^2-1}{3}$, and $\E(W(\pi))= \frac{1}{2}{n+1 \choose 3}$ (see Lemma~\ref{lem:upperbounds-fitness}).
Thus, for a fixed $p$, the algorithm achieves only a multiplicative constant improvement over a random permutation.
Similarly to the case where only adjacent elements are swapped, experiments of $S_{p,n}$ in Section~\ref{sec:conv-quality} indicate that $\E(I(S^{\;\infty}_{p,n}))\simeq f_3(p)\cdot n^2$ and $\E(W(S^{\;\infty}_{p,n})\simeq f_4(p)\cdot n^3$ for some (unknown) functions $f_3$ and $f_4$. In particular, for $p\to 1/2$, both $\E(I)$ and $\E(W)$ smoothly converge to their expected values for a random permutation.


\subsubsection{Bounded-range swaps}

Since both the cases $r=1$ and $r=n$ have severe drawbacks, we also study the intermediate range, $1 < r < n$. Indeed, experiments suggest a smooth transition between the above results, namely $\E(I(\pi))= \Theta(nr)$ and $\E(W(\pi)) = \Theta(nr^2)$ for any fixed $p\leq 0.3$ (see for instance Figure~\ref{fig:I-by-r}). Thus is seems that intermediate values of $r$ allow to find a compromise between high fitness (for smaller $r$) and low convergence time (for larger $r$). Again we do not provide theoretical results on the convergence time, but we do prove lower bounds on the average fitness that make the experimental findings.


\begin{restatable}[Bounded-Range Swaps]{theorem}{TheoremFour}\label{thm:range-swap}
	For any $0<p<{1}/{2}$ and any $r\!\in\!\{1,\dots,n\}$, 
	\begin{align*}
	\E(I(S_{p,r}))=  \Omega(prn),&&
	\E(D(S_{p,r}))=  \Omega(prn),&&
	\E(W(S_{p,r}))=  \Omega(pr^2n).
	\end{align*}
\end{restatable}

It remains an open problem to show that these lower bounds are tight, as experiments suggest.


\subsection{Relation to the Conference Version}
An extended abstract of this work has appeared in the Proceedings of the Genetic and Evolutionary Computation Conference (GECCO 2017)~\cite{gavenvciak2017sorting}. In comparison, the present version gives more background information (Sections~\ref{sec:relations},~\ref{sec:markov}, and~\ref{sec:fitnessfunctions}), and it gives relations between the three different fitness functions (Lemmas~\ref{lem:IandW},~\ref{lem:I2andW},~\ref{lem:DandW}) that have so far been missing in the literature since they have not been studied concurrently before. We have also extended the experimental results in Section~\ref{sec:experimental} by Figures on $W$. Furthermore, all theorems now contain explicit dependencies on $p$: While $p$ was assumed to be a constant in the GECCO version, the results in this version also apply to varying $p=p(n)$. In particular we allow the case $p \to 0$.
Finally, we corrected and tightened one upper bound in Theorem~\ref{thm:adj-quality} (a special thank for the reviewer who adverted to this issue).

\subsection{Outline}
The rest of this work is structured as follows. In Section~\ref{sec:formal} we first provide a short introduction to Markov chains, afterwards we formally define the three fitness functions and show some of their properties and relations. Then we turn to the theoretical analysis of our Markov chain in Section~\ref{sec:theory} and prove the results stated above. Finally, we present some experimental results in Section~\ref{sec:experimental}.

\section{Notation and Formal Definitions}\label{sec:formal}

\subsection{Preliminary definitions on Markov chains}\label{sec:markov}

We introduce some definitions on \textit{finite Markov chains}, which are used in this work (see Levin et al.~ \cite{LevPerWil09} for more details).
%
A finite Markov chain is a process that operates on a \textit{finite state space} $\mathcal{S}$ and is specified by a \textit{transition matrix} $P$ as follows:
for any two states $s$ and $s'$ of $\mathcal{S}$, $P(s,s')$ gives the probability of going from $s$ to $s'$ in one step.
Consequently, the probability of going from one state to another state in $t$ steps is given by the $t^{th}$ power of the transition matrix.

The term $P^t(s,\cdot)$ gives the probability distribution over all states after $t$ steps, when starting in state $s$.
If a probability distribution $q$ over the states in $\mathcal{S}$ satisfies
\[q=qP\, ,\]
then we say that $q$ is a \textit{stationary distribution} of the Markov chain. 

A Markov chain is \emph{irreducible} if it can eventually reach every state from every state, that is for all states $s$ and $s'$ there exists a $t$ such that $P^t(s,s')>0$. Every irreducible chain has a \textit{unique} stationary distribution (Corollary 1.17 in \cite{LevPerWil09}).
Furthermore, if for all states $s$ of an irreducible chain, $P(s,s)>0$, the chain is called \emph{aperiodic}, and every irreducible aperiodic chain eventually converges to its unique stationary distribution  $q$: For any two states $s$ and $s'$,  $\lim_{t\rightarrow \infty} P^t(s,s')=q(s')$ (Theorem 4.9 (Convergence Theorem) in  \cite{LevPerWil09}).


If there is a probability vector $q$ such that
for all $s$ and $s'$, \[q(s) P(s,s') = q(s')P(s',s)\, ,\] we say that $q$ satisfies the \textit{detailed balance condition}.
In this case, $q$ is a {stationary distribution} of the corresponding Markov chain and we say that the chain is \emph{reversible} (Proposition 1.20 in \cite{LevPerWil09}). 

An equivalent characterization of reversibility is given by looking at cycles over the states. A Markov chain is reversible if for a given starting state $s$, every sequence of transitions (i.e., steps from one state to another) that returns to $s$, thus forms a cycle, has the same probability whether it is followed in one direction or the other.
The \textit{Kolmogorov reversibility criterion} (see \cite{kelly1979reversibility}, Chapter 1.5, Theorem 1.7) states that a chain is reversible if and only if for any cycle $\mathcal{C}$, that is a sequence of transitions $((s,s'), (s',s''), \dots, (s''',s))$, it holds
\[\prod_{(s,s')\in \mathcal{C}}P(s,s') = \prod_{(s,s')\in \mathcal{C}}P(s',s)\, .\] 

To measure the similarity of two probability distributions $u$ and $v$ over $\mathcal{S}$, we compute their \textit{total variation distance}, i.e.,
\[||u-v||_{TV} = \frac{1}{2}\sum_{s\in\mathcal{S}}|u(s)-v(s)|\, .\]
This distance is used to compute the \textit{mixing time} of a Markov chain, which is defined as the smallest time $t$ needed, such that for any starting state $s$, $ |P^t(s,\cdot) - q| $ is sufficiently small:
\begin{align*}
\label{eq:t_mix}
T_{\text{mix}}(\epsilon) := \min\{t \in \mathbb{N}\mid \max_{s\in \mathcal{S}}\{ ||P^t(s,\cdot) - q||_{TV}\} \leq \epsilon\}\, .
\end{align*}

\subsection{Fitness Functions and their Properties}\label{sec:fitnessfunctions}

We consider the following three functions, which have been already mentioned in the introduction, as fitness functions to be minimized: The total dislocation in a permutation $\pi$ is the sum of displacement of all elements, where the displacements of an element is the absolute difference between its positions in $\pi$ and in the sorted $n$-element sequence $\pi_0 = (1,\dots,n)$; the number of inversions is the number of pairs of elements in $\pi$ that are placed in a different order than in $\pi_0$; the weighted number of inversions is the sum of all absolute differences in value of inverse pairs. The formal definitions of the three fitness functions are as follows:

\begin{definition}\label{def:fitnessfunctions}
	Let $\pi \in \mathcal{S}_n$ be a permutation of the set $\{1,\dots,n\}$.
	%
	The \emph{number of inversions} of $\pi$ 
	is  $$I(\pi) := \sum_{i<j \colon \pi(i)>\pi(j)}{1}\ .$$
	The \emph{weighted number of inversions} of $\pi$ 
		is  $$W(\pi) := \sum_{i<j \colon \pi(i)>\pi(j)}{j-i}\ .$$
	The \emph{total dislocation} (or \emph{Spearman's footrule}) of $\pi$
	is  $$D(\pi) := \sum_{i\in\{1,\dots,n\}}{\lvert i-\pi(i)\rvert}\ .$$
\end{definition}

\begin{remark}\label{rem:sequences}
	For general $n$-element sequences $s$ (even with duplicates), the number of inversions and the weighted number of inversions are equivalently given by
	$I(s) :=\sum_{i<j \colon \pi_i>\pi_j}{1}$, and $W(s) := \sum_{i<j \colon \pi_i>\pi_j}{\pi_i-\pi_j}$, respectively. 
\end{remark}

Note that any of these fitness functions decreases with a successful swap. Therefore, the function $\Swapsort_{p,r}$ can equivalently be defined by first creating a new mutation $\pi'$ from the current permutation $\pi$ by a random swap of two indices of distance at most $r$, then comparing the fitness of $\pi$ and $\pi'$, and returning the fitter with probability $1-p$, and the less fit with probability $p$. We used the previous definition to emphasize that $\IsFitter$ can be defined without reference to any explicit fitness function.




The (non-absolute) dislocations of the single elements permit equivalent ways to express the weighted number of inversions in any permutation $\pi\in\mathcal{S}_n$:
	\begin{lemma}\label{lem:equivalence}
		For any $\pi\in \mathcal{S}_n$,
		\[W(\pi) = \sum_{i\in\{1,\dots,n\}}i(i-\pi(i)) = \frac{1}{2}\sum_{i\in\{1,\dots,n\}}(i-\pi(i))^2\, .\]
	\end{lemma}
	\begin{proof}
		The first equality has been shown in \cite{barbaraPaper}:
		In the sum $\sum_{i<j \colon \pi(i)>\pi(j)}{j-i}$, each element $i$ is 
		added $s_i$ and subtracted $l_i$ times, where $s_i$ is the number of smaller elements on the right of $i$ and $l_i$ the number of larger elements on its left. The difference $d_i = s_i -l_i$ is equal to $i$'s dislocation to the left, i.e., $d_i = i-\pi(i)$. Thus, $i \cdot d_i$ is exactly the contribution of $i$ to $W(s)$.
		
		The second equality follows immediately by the observation that $$\sum_{i\in\{1,\dots,n\}}i^2= \sum_{i\in\{1,\dots,n\}}(\pi(i))^2\, ,$$ which implies \[\sum_{i\in\{1,\dots,n\}}(i-\pi(i))^2  = \sum_{i\in\{1,\dots,n\}}2i^2-2i\pi(i)\, . \]
\end{proof}

To provide the reader with a feeling for the different fitness functions, the following lemma provides for each of them a tight upper bound. We also give the expected value for a random permutation.
%
\begin{lemma}\label{lem:upperbounds-fitness}
	For any $\pi \in \mathcal{S}_n$,
	\begin{align*}
	I(\pi) \leq \binom{n}{2}, &&
	W(\pi)  \leq \binom{n+1}{3}, &&
	D(\pi)  \leq \left\lfloor\frac{n^2}{2}\right\rfloor.
	\end{align*}
	For $\pi$ chosen uniformly at random from $\mathcal{S}_n$,
	\begin{align*}
	E(I(\pi)) = \frac{1}{2}\cdot\binom{n}{2}, &&
	E(W(\pi))  = \frac{1}{2}\cdot\binom{n+1}{3}, &&
	E(D(\pi))  = \frac{n^2-1}{3}.
	\end{align*}
\end{lemma}
\begin{proof}
	We first show the upper bounds. The number of different pairs in a set of $n$ elements is $\binom{n}{2}$. Therefore, the maximum number of inversions in a permutation is $\binom{n}{2}$. 
	The total sum of the absolute differences of all pairs in the set $\{1,\dots,n\}$ is $\sum_{k=1}^{n-1}(n-k)k = \binom{n+1}{3}$:
	For every $k\in\{1,\dots,n-1\}$, there are exactly \mbox{$n-k$} pairs with absolute difference $k$. The identity is trivial for $n=2$ and follows by induction for larger $n$, i.e.,
	\[\sum_{k=1}^{n-1}(n-k)k = \sum_{k=1}^{(n-1)-1}((n-1)-k)k + \sum_{k=1}^{n-1}k = \binom{n}{3}+\binom{n}{2} = \binom{n+1}{3}\, .\] 
	The maximum total dislocation of a permutation on $n$ elements is $\lfloor\frac{n^2}{2}\rfloor$ (see~\cite{Mitchell04}, Corollary 2.3). 
	These three upper bounds are all tight for the reversed sorted permutation $(n,\dots,1)$. 
	
	In a random permutation of $n$ elements, the expected number of inversions is $\frac{1}{2}{n\choose 2}$, since any pair of elements is inverse with probability one half. For the same reason, the expected weighted number of inversions is $\frac{1}{2}{n+1\choose 3}$. The expected total dislocation is $\frac{n^2-1}{3}$ (see again \cite{Mitchell04}, Remark 2.7).
\end{proof}

\medskip
	We continue with a handful of non-trivial upper and lower bounds that set the three fitness functions into relation.
	There are two existing results for the number of inversions and total dislocation. We use $Ex(\pi)$ to denote the minimum number of successive (any pair) swaps needed to sort a permutation $\pi$:
	\begin{lemma}[Diaconis and Graham 1977 \cite{diaconis1977}]\label{lem:diaconis}
		For any $\pi\in \mathcal{S}_n$,
		\[I(\pi) + Ex(\pi) \leq D(\pi) \leq 2\cdot I(\pi)\, .\]
	\end{lemma}
	\begin{lemma}[Hadjicostas and Monico 2015 \cite{hadjicostas2015}]
		For any $\pi\in \mathcal{S}_n$,
		\[D(\pi) \leq I(\pi) + Ex(\pi) + \lfloor{n}/{2}\rfloor\left(\lfloor{n}/{2}\rfloor -1\right)\, .\]
	\end{lemma}
	
	While the three fitness functions do not determine each other, a small (large) value in one of them does determine upper (lower) bounds on the other two. The following three lemmas collect inequalities between the different fitness functions. Apparently these bounds have not been published before, probably because previous work used to pick only one function. We first state all three lemmas, and give the proofs afterwards.
	\begin{lemma}\label{lem:IandW}
		For any $\pi\in \mathcal{S}_n$,
		\[I(\pi) \leq W(\pi) \leq \frac{n}{2}\cdot I(\pi)\, .\]
	\end{lemma}
	\begin{lemma}\label{lem:I2andW}
		For any $\pi\in \mathcal{S}_n$,
		\[I(\pi)^2 \leq 2n\cdot W(\pi) \, .\]
	\end{lemma}
	\begin{lemma}\label{lem:DandW}
		For any $\pi\in \mathcal{S}_n$,
		\[\frac{1}{2}\cdot D(\pi) \leq W(\pi) \leq \frac{n-1}{2}\cdot D(\pi)\, .\]
	\end{lemma}
	\begin{proof}[of Lemma \ref{lem:IandW}]
		Since every inversion weighs at least 1, the lower bound is trivial. It is tight if and only if only elements that differ by one are inverted. 
		There is also a permutation for which the upper bound is tight, namely $\pi=(n,1,2,\dots,n-1)$.
		
		We now show the upper bound by induction on $n$ with the claim straightforward for $n\leq 2$.
		Assume that 
		\[\frac{W(\pi^{[n]})}{I(\pi^{[n]})} \leq \frac{n}{2}\]
		holds for any permutation $\pi^{[n]}\in\mathcal{S}_n$.  Let $\pi^{[n+1]}$ be any permutation in $\mathcal{S}_{n+1}$ and let $k$ be the position of element $(n+1)$.
		
		Moving $(n+1)$ to position $n+1$ by shifting the elements between $\pi^{[n+1]}_{k+1}$ and $\pi^{[n+1]}_{n+1}$ each left by one position, results in a permutation $\pi^{[n]}$ with $(n+1)$ appended.
		We observe the changes in $W$ and $I$: 
		\begin{align}\label{eq:changes}
		\Delta(W) := W(\pi^{[n+1]}) - W(\pi^{[n]})\, , && \Delta(I) :=  I(\pi^{[n+1]}) - I(\pi^{[n]}).
		\end{align}
		Performing an adjacent swap of $(n+1)$ with every element at positions $k+1$ to $n+1$, decrements the number of inversions in every swap since $(n+1)$ is the largest element. Thus, 
		\begin{equation}\label{eq:deltaI}
		\Delta(I) = n+1-k.
		\end{equation}
		To upper bound $\Delta(W)$ we use that $(n+1)$ is shifted to the right by $n+1-k$ positions and $n+1-k$ elements are shifted to the left by one position. By Lemma \ref{lem:equivalence},
		\begin{equation}\label{eq:deltaW}
		\Delta(W) = (n+1-k)(n+1) - \sum_{i=k+1}^{n+1}\pi_i^{[n+1]}.
		\end{equation}			
		Let $\Sigma := \sum_{i=k+1}^{n+1}\pi_i^{[n+1]}$ and observe that 
		\begin{equation}\label{eq:sigma}
		\Sigma_{\min} :={n+2-k \choose 2}\leq \Sigma \leq{n+1 \choose 2} - {k\choose 2}=:\Sigma_{\max},
		\end{equation}
		since these elements have a value between 1 and $n$, and the $n+1-k$ smallest (largest) sum up to $\Sigma_{\min}$ ($\Sigma_{\max}$). Note that \begin{equation}\label{eq:minmax}
		\Sigma_{\min} + \Sigma_{\max} = (n+1)(n+1-k).
		\end{equation} 
		Before we continue the proof, we shall show the following claim.
		\begin{claim} For any permutation $\pi^{[n]}\in\mathcal{S}_n$ and $k\geq 1$, it holds for the sum of the last $n+1-k$ elements $\Sigma$ that
			\begin{equation}\label{eq:deltamax}
			\Sigma \geq \Sigma_{\max}  - I(\pi^{[n]})\, .
			\end{equation}
		\end{claim}
		\begin{proof}
			Split the elements between positions $k-1$ and $k$ into a left and a right set, and count for each element on the right its inversions with elements on the left:
			\begin{align*}
			I_i^k :=			
			\lvert\{j < k \mid \pi_j^{[n]} > \pi_i^{[n]}\}\rvert && \forall i\ge k \, .
			\end{align*}			
			The sum of all $I_i^k$ is equal to the number of inversions between the first $k-1$ and the last $n+1-k$ elements. Thus, $\sum_{i=k}^{n} I_i^k \leq I(\pi^{[n]})$.
			
			Moreover, the sum of all $I_i^k$ is exactly what we lose in $\Sigma$ compared to $\Sigma_{\max}$.
			To see this, assume that, w.l.o.g., the elements of the right set are sorted increasingly. In this case, $\pi_i^{[n]} + I_i^k = i$, for all $k\le i\le n$, and consequently
			\begin{align*}
			\Sigma + \sum_{i=1}^{n} I_i^k =\sum_{i=k}^n(\pi_i^{[n]} + I_i^k) &= \sum_{i=k}^n i = \Sigma_{\max}\, .
			\end{align*}
		\end{proof}
		\noindent\textit{Continuing the Proof of Lemma~\ref{lem:IandW}.} By the induction hypothesis, $W(\pi^{(n)}) \leq {n}/{2}\cdot I(\pi^{(n)})$. Using Equations \eqref{eq:changes} and \eqref{eq:deltaI} we write:
		\begin{align}
		W(\pi^{(n)}) &\leq \frac{n}{2}\cdot I(\pi^{(n)})\nonumber\\
		&= \frac{n+1}{2}\cdot I(\pi^{(n)}) - \frac{1}{2}\cdot I(\pi^{(n)})\nonumber\\
		&= \frac{n+1}{2}\cdot I(\pi^{(n+1)}) - \frac{(n+1)(n+1-k)}{2} - \frac{1}{2}\cdot I(\pi^{(n)})\, .\label{eq:IH}
		\end{align}
		Finally, we bound $W(\pi^{(n+1)})$ using Equations \eqref{eq:changes}, \eqref{eq:deltaW}, \eqref{eq:sigma}, \eqref{eq:minmax}, \eqref{eq:deltamax}, and \eqref{eq:IH}:
		\begin{align*}
		W(\pi^{(n+1)}) &= \Delta(W) + W(\pi^{(n)})\\
		&\leq (n+1-k)(n+1) - \Sigma \\
		&\phantom{=~~} + \frac{n+1}{2}\cdot I(\pi^{(n+1)}) - \frac{(n+1)(n+1-k)}{2} - \frac{1}{2}\cdot I(\pi^{(n)})\\
		&= \frac{\Sigma_{\min} + \Sigma_{\max}}{2} - \Sigma + \frac{n+1}{2}\cdot I(\pi^{(n+1)}) - \frac{1}{2}\cdot I(\pi^{(n)})\\
		&\leq \frac{\Sigma_{\min} + \Sigma + I(\pi^{(n)})}{2} - \Sigma + \frac{n+1}{2}\cdot I(\pi^{(n+1)}) - \frac{1}{2}\cdot I(\pi^{(n)})\\
		&\leq \frac{n+1}{2}\cdot I(\pi^{(n+1)})\, .
		\end{align*}
		This concludes the proof.
	\end{proof}
	\begin{proof}[of Lemma \ref{lem:I2andW}]
		Let $I_i(\pi)$ denote the number of smaller elements than $i$ among the positions $\pi(i)+1$ to $n$. Then, $\sum_{i=1}^{n}I_i(\pi) = I(\pi)$.
		Similarly, let $W_i$ be analogue to $I_i$ but for weighted inversions:
		\[W_i(\pi) := \sum_{j<i \colon \pi(j) > \pi(i)}i-j.
		\]
		Since all elements are unique (by our assumption on the ground set), no two elements $j<i$ have the same difference to $i$. Therefore,
		$
		W_i(\pi) \geq {I_i(\pi)+1 \choose 2},
		$
		which implies that 
		\begin{align}
		W(\pi) = \sum_{i=1}^n W_i(\pi) 
		\geq \sum_{i=1}^n {I_i(\pi)+1 \choose 2} \geq \frac{1}{2}\sum_{i=1}^n I_i(\pi)^2\,.\label{eq:wixi}
		\end{align}
		By the relation between the arithmetic and the quadratic mean of $n$ values it holds that
		$
		\sqrt{{\sum_{i=1}^n I_i(\pi)^2}/{n}} \geq {\sum_{i=1}^n I_i(\pi)}/{n} = {I(\pi)}/{n}\,,
		$
		which we can rewrite as
		\begin{equation}
		\sum_{i=1}^n I_i(\pi)^2 \geq \frac{I(\pi)^2}{n}.\label{eq:mean-rewritten}
		\end{equation}
		Finally, by Equations \eqref{eq:wixi} and \eqref{eq:mean-rewritten} we conclude 
		$W(\pi)\geq\frac{I(\pi)^2}{2n}$.
	\end{proof}
	\begin{proof}[of Lemma \ref{lem:DandW}]
		The lower bound is implied by Lemmas \ref{lem:diaconis} and \ref{lem:IandW}, and it is tight if and only if there are only inversions between elements that differ by one. 
		
		For the upper bound, consider the following procedure that sorts $\pi$: While $\pi$ is not sorted, swap two inverse elements $\pi_i>\pi_j$ with $i<j$, such that their difference $d = \pi_i-\pi_j $ is maximum among all inversions.
		By this maximality, $\pi_j<\pi_k<\pi_i$ for any $i<k<j$. Therefore, $\pi_k$ was inverse to both $\pi_i$ and $\pi_j$ and these inversions are removed by the swap, while no new inversions are introduced. Let $m = j-i-1$ denote the number of these elements and note that $m\leq d-1$. 
		For each such $\pi_k$, the two inversions that are removed weighted $(\pi_i-\pi_k)-(\pi_k-\pi_j)=(\pi_i-\pi_j)=d$. Also the inversion between $\pi_i$ and $\pi_j$ itself weighted $d$.
		Thus, the decrease of the total weighted inversion\footnote{Observe that the decrease of the total number of inversions is $(1+2m)$.} is $d(m+1)$.
		By the maximality of $d$, it holds that $\pi_j<i$ and $\pi_i>j$. Thus the total dislocation decreases by $2(m+1)$ for $d(m+1)$ decrease in $W$, or $2$ for $d$ independent of $m$. We obtain the claim with $d\leq n-1$. The upper bound is tight for $\pi=(n,2,\dots,n-1,1)$.
	\end{proof}

\section{Theoretical analysis}\label{sec:theory}


\subsection{General properties of the Markov chain}\label{sec:general-properties}
We first show that all Markov chains that we consider in this work are irreducible and aperiodic. This then implies (by Corollary 1.17 and Theorem 4.9 in  \cite{LevPerWil09}) that they converge to a unique stationary distribution.
Note that we show this property for general $n$-element sequences, because in Section~\ref{sec:adj}, we consider $S_{p,1}$ also on 0-1 sequences.


\begin{lemma}\label{lem:unique}
	For any $n$-element sequence $s$, the Markov chain $S_{p,r}(s)$ is aperiodic and irreducible.
\end{lemma}
\begin{proof}
	The  chain $S_{p,r}$ is aperiodic, since in each transition, the state does not change with a positive probability:
	Observe that $\Swapsort_{p,r}$ always makes a swap with probability at most $1-p$, independent whether the swap is bad or good. Therefore, it is true that in the transition matrix $P$, $P(s,s)\ge p>0$.
	
	The chain is irreducible, since one can achieve every permutation $s'$ of the elements of $s$ by a sequence of adjacent swaps:
	Every element in $s$ can move to any position by repeated swaps with its left or right neighbour. Thus, the elements of $s'_1, s'_2, \dots, s'_n$ in $s$ can iteratively move to their target positions 1 to $n$, respectively. Since every element moves by at most $\MO(n)$ positions, there is a $t\in\MO(n^2)$, such that $P^t(s,s')>0$.    
\end{proof}

In the rest of this section, we will first analyze the sorting algorithm with adjacent swaps (Section~\ref{sec:adj}), then the sorting algorithm where we allow any swaps (Section~\ref{sec:any}), and finally the algorithm with bounded range swaps (Section~\ref{sec:bounded}).


\subsection{Sorting with adjacent swaps}\label{sec:adj}

Here we theoretically analyze $S_{p,1}$, the sorting algorithm with adjacent swaps, and its Markov chain $S_{p,1}$. A main tool will be a sorting process of 0-1 sequences, which is analyzed in Section~\ref{sec:01}, and coupled to $S_{p,1}$ in Section~\ref{sec:proofadjacent}.

\medskip
We first show that for any $0<p<1/2$ and any $n>1$, the process $S_{p,1}(s)$ on any sequence $s=(s_1,\dots,s_n)$ is reversible, which implies together with Lemma~\ref{lem:unique} that its stationary distribution is unique. In the following lemma, we view $S_{p,1}$ as a process on arbitrary sequences in order to get the same result also for 0-1 sequences.

\begin{lemma}\label{lem:adj-reversible}
	For any $0<p<1/2$ and any $n$-element sequence $s$, the Markov chain $S_{p,1}(s)$ is reversible and has a unique stationary distribution $q$ with
	\[q(s')=\frac{1}{Z}\cdot\left(\frac{p}{1-p}\right)^{I(s')}\ ,\]
	for any permutation $s'$ of $s$, where $Z$ is a normalizing constant for the distribution (only depending on $p$ and the composition of $s$).
\end{lemma}
\begin{proof}
	Let $P$ be the transition matrix of $S_{p,1}(s)$. The chain is reversible if all states $s$ and $s'$ satisfy the detailed balance condition:
	$q(s')\cdot P({s',s}) = q(s)\cdot P({s,s'})$.
	If $s$ and $s'$ differ by more than one adjacent swap, then $P({s',s})=P({s,s'})=0$.
	Otherwise, without loss of generality assume $I(s') = I(s)+1$. Then the probabilities to go from $s$ to $s'$ and vice versa are $P({s,s'}) = {p}/{(n-1)}$ and $P({s',s}) = {(1-p)}/{(n-1)}$. These satisfy the detailed balance condition from above. 
	
	By Proposition 1.19 in Levin et al.~\cite{LevPerWil09}, $q$ is thus a stationary distribution for the Markov chain $S_{p,1,n}$. The uniqueness of $q$ follows from irreducibility and aperiodicity (Lemma~\ref{lem:unique}).
\end{proof}

\subsubsection{Sorting 0-1 sequences}\label{sec:01}

Let $B_{p,n,k}=S_{p,1}(0^k1^{n-k})$, e.g., the sorting process of $k$ zeroes and $n-k$ ones, and let $B_{p,n,k}^{\;\infty}$ denote its stationary distribution. The Markov process is irreducible, aperiodic and with a unique stable distribution by Lemmas~\ref{lem:unique} and~\ref{lem:adj-reversible} (with the preceding remark). Since the stationary distribution is unique, the starting state choice is just a convenience. 

\begin{lemma}\label{lem:zero-one-quality}
	Let $p_{\max}<{1}/{3}$ be constant. Then for any $0<p <p_{\max}$, and any $n$ and $0<k<n$, we have
	\[\E(I(B^{\;\infty}_{p,n,k}))\leq\frac{2p}{(1-3p)}+2^{-\Omega(n)}\ .\]
	Moreover, for some $\lambda(p)$ depending only on $p$, any $l\geq 0$, and  $\bar{c}=2p/(1-p)$,
	\[\MP[I(B^{\;\infty}_{p,n,k})>\lambda(p)+l]<\bar{c}\,^l\ .\]
\end{lemma}
\begin{proof}
	For a 0-1 sequence $s=(s_1,\dots,s_n)$, let $u=\lvert\{s_i \mid s_i<s_{i+1}\land 1\leq i\leq n-1\}\rvert$ and $d=\lvert\{s_i \mid s_i>s_{i+1}\land 1\leq i\leq n-1\}\rvert$ be the number of \emph{up}-transitions and \emph{down}-transitions of $s$, respectively. Obviously, $d-1\leq u\leq d+1$.
	
	In the state with sequence $s$, the probability that $I(s)$ increases in one step by 1 is $p^\uparrow=\frac{up}{n-1}$, since this happens only when an up-transition pair is selected and the elements are sorted as descending. Analogously, the probability $I(s)$ decreases by 1 is $p^\downarrow=\frac{d(1-p)}{n-1}$. Note that $I(s)$ may not change by more than 1 in a single step.
	
	For the sequence of values of $I$, we can observe that whenever $d>0$ and $u>0$ it holds that
	\[\frac{p^\uparrow}{p^\downarrow}=\frac{up}{d(1-p)}\leq \frac{(d+1)p}{d(1-p)}\leq\frac{2p}{1-p}\ .\]
	The only state with $d=0$ is the sorted state with $I=0$ and the only state with $u=0$ is the reversed sorted state with $I=I_{\max}=k(n-k)$, i.e., all ones on the left and all zeros on the right, such that each of the $k$ ones has $n-k$ inversions. In all other states, the ratio $p^\uparrow/p^\downarrow$ depends only on $u$ and $d$ and is upper bounded by $2p/(1-p)$.
	
	\medskip
	To upper bound $\E(I)$, consider the random walk $\bar{I}$ on $0,\dots,I_{\max}$ with $\bar{p}^\uparrow=2p/(p+1)$ and $\bar{p}^\downarrow=(1-p)/(p+1)$, so $\bar{p}^\uparrow+\bar{p}^\downarrow=1$. In the stationary distribution of $\bar{I}$, $\MP[\bar{I}=i]=\bar{c}\,^i/Z$, where $\bar{c}=\bar{p}^\uparrow/\bar{p}^\downarrow=2p/(1-p)$ and $Z=(1-\bar{c}^{I_{\max}+1})/(1-\bar{c})$ is a normalizing constant. This follows directly by analyzing the ratios $\MP[\bar{I}=i]/\MP[\bar{I}=i+1]$.
	
	By a direct series summation we get
	\begin{align*}
	\E(\bar{I})&=\sum_{i=0}^{I_{\max}}i\cdot\bar{c}\,^i\cdot\frac{1}{Z}
	\\&= \frac{\bar{c}\cdot(1+I_{\max}\cdot \bar{c}^{I_{\max}+1}-(I_{\max}+1)\cdot\bar{c}^{I_{\max}})}{(1-\bar{c})^2}\cdot\frac{(1-\bar{c})}{(1-\bar{c}^{I_{\max}+1})}\\	
	&=\frac{\bar{c}}{(1-\bar{c})}\cdot\frac{(1\pm \MO(I_{\max}\cdot\bar{c}^{I_{\max}} ))}{(1-\bar{c}^{I_{\max}+1})}\, .
	\end{align*}
	Using $n-1\le I_{\max}< n^2$ and $0<\bar{c}<1$, as well as $\frac{1}{(1-\bar{c}^{I_{\max}+1})} 
	\le (1+\bar{c}^{\Omega(n)})$, we conclude
	\[\E(\bar{I}) = \frac{2p}{1-3p}\pm2^{-\Omega(n)}\, .\]
		
	Since $p^\uparrow/p^\downarrow\leq\bar{p}^\uparrow/\bar{p}^\downarrow$, the random walk $\bar{I}$ stochastically dominates the sequence of values of $I$, and in particular, we get $\E(I)\leq\E(\bar{I})$ in the stationary distributions.
	
	\smallskip
	For the second part of the lemma, note that \[\MP[\bar{I}>l]\leq\sum_{i=l+1}^{\infty} \bar{c}^i/Z=\bar{c}\,^{l+1}/((1-\bar{c})Z)=\bar{c}\,^l\bar{c}\,^{1-\log_{\bar{c}}((1-\bar{c})Z)}\, .\]
	With $\lambda(p)=\log_{\bar{c}}((1-\bar{c})Z)$ we get $\MP[\bar{I}>\lambda(p)+l]<\bar{c}^l$, and the bound for $I$ follows from being stochastically dominated by $\bar{I}$.
\end{proof}

\subsubsection{Sorting permutations (Proof of Theorem~\ref{thm:adj-quality})}\label{sec:proofadjacent}

We recall Theorem~\ref{thm:adj-quality}:
\TheoremOne*

\begin{proof}	
	Let $\pi_0=(1,2,\dots, n)$ be the sorted $n$-element sequence. By Lemma~\ref{lem:adj-reversible}, the stationary distribution $S^{\;\infty}_{p,1,n}$ is the same as $S^{\;\infty}_{p,1}(\pi_0)$, so we analyze the latter. For a permutation $\pi$, let $T_k(\pi)\in\{0,1\}^n$ be the \emph{$k$-th threshold 0-1 sequence}, where $T_k(\pi)_i=1$ if $\pi_i>k$ and $T_k(\pi)_i=0$ if $\pi_i\leq k$. So $T_k(\pi)$ contains exactly $k$ zeroes and $n-k$ ones.
	
	We decouple the process $S_{p,1}(\pi_0)$ into $n-1$ processes $(B_{p,n,1},\dots, B_{p,n,n-1})$: The state $\pi$ of the process $S_{p,1}(\pi_0)$ corresponds to the states $(T_1(\pi), \dots, T_{n-1}(\pi))$ of the 0-1 process $B_{p,n,k}$. The coupled processes share the following event space: In every step we randomly choose two adjacent positions $(i,i+1)$ to be compared and with probability $p$ choose to order their values in the descending (wrong) order, or ascending (right) order otherwise. The same event then decides the change in all the coupled processes. It is straightforward to see that starting from corresponding states of $S_{p,1}$ and $(B_{p,n,1},\dots B_{p,n,n-1})$, the resulting states after one event are again corresponding. See Figure~\ref{fig:decoupling} for an illustration.
	
	\begin{figure}
		\begin{center}
			\includegraphics[width=0.45\columnwidth]{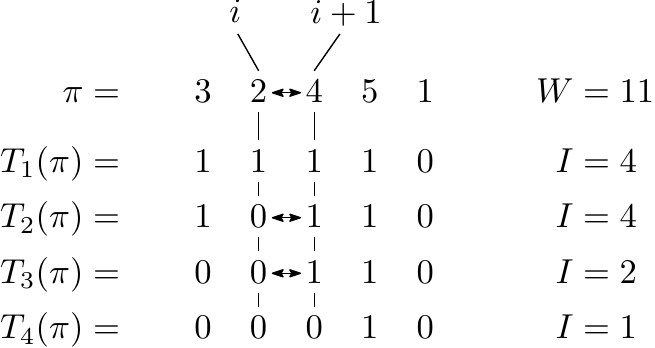}
		\end{center}
	\caption{A single step of the coupled processes.} \label{fig:decoupling}
	\end{figure}
	
	Now observe that an inversion of weight $w$ in a permutation $\pi$ is present in exactly $w$ of the coupled binary strings: Assume that $\pi_i>\pi_j$ and $i<j$, then $T_k(\pi)_i = 1$ and $T_k(\pi)_j=0$ for every $k\in\{\pi_j+1,\dots,\pi_i\}$.
	Thus, for any permutation $\pi$, we have $W(\pi)=\sum_{k=1}^{n-1}I(T_k(\pi))$ by reordering the summation:
	\begin{eqnarray}
	W(\pi)=\sum_{\substack{i<j\colon\\ \pi_i>\pi_j}}\pi_i-\pi_j=\sum_{k=1}^{n-1}\sum_{\substack{i<j\colon\\ \pi_i>k\\ \pi_j\leq k}}1=\sum_{k=1}^{n-1}I(T_k(\pi))\ .\label{eq:w-decom}
	\end{eqnarray}
	
	In the stationary distribution of the coupled processes, we have $\E(W(S^{\;\infty}_{p,1,n}))=\sum_{k=1}^{n-1}\E(I(B^{\;\infty}_{p,1,k}))$.
	By Lemmas~\ref{lem:IandW}~and~\ref{lem:zero-one-quality}, we conclude that 
	\[\E(I(S^{\;\infty}_{p,1,n}))\le\E(W(S^{\;\infty}_{p,1,n}))\le n
    \cdot\frac{2p}{(1-3p)}+2^{-\Omega(n)}\, ,\]
	which proves the upper bound of the first part of the theorem.
	
	We show the lower bound on $\E(I(S^{\;\infty}_{p,1,n}))$ to hold after any number of steps: Consider any finite length realization of the random process starting from a random permutation $\pi_0$ with the last permutation being $\pi$. Divide the $n-1$ pairs of adjacent positions into three groups: $A$ is the set of pairs that have never had a swap, $B$ is the set of pairs whose last swap was bad, and $C$ is the set of pairs whose last swap was good. It is easy to see that $\E(|C|)\leq(1-p)(n-1)$ and so $\E(|A|+|B|)\geq p(n-1)$. In $\pi$ we find one inverted pair for each $B$: The pair that was last swapped in $B$ has to be also inverted in $\pi$. For each position in $A$, the expected number of inversions of $\pi_0$ between the elements left of $A$ and right of $A$ is at least 1 whenever $n\geq 3$, and all these inversions are also present in $\pi$. The claim follows from linearity of expectation and the observations that all the inverted pairs counted for $A$ and $B$ are distinct.
	
	\smallskip
	To see the upper bound on $W(S^{\;\infty}_{p,1,n})$ that holds with high probability, we use the tail bounds of Lemma~\ref{lem:zero-one-quality}: For any $k$ and $\beta>0$, \[\MP[I(B^{\;\infty}_{p,n,k})>\lambda(p)+\beta\log n]<\bar{c}\,^{\beta\log n}.\] For $\beta$ big enough we get $\bar{c}\,^{\beta\log n}=\MO(n^{-3})$ and by a union bound over $n-1$ events, $I(B^{\;\infty}_{p,n,k})>\lambda(p)+\beta\log n$ for $k=1,\dots,n-1$. Together with Equation~\eqref{eq:w-decom}, this implies the bound on $W$. 
	The bound for $I$ follows from Lemma~\ref{lem:IandW}.
	
	Finally, we show the bound on the maximum dislocation: Consider any permutation $\pi$, any element $1\leq k\leq n$, and let $j=\pi(k)$. The dislocation of $k$ is thus $|k-j|$. If $j<k$, then $T_k(\pi)_j=1$ and necessarily $I(T_k(\pi))\geq k-j$. By Lemma~\ref{lem:zero-one-quality}, $\MP[k-j>\beta\log n]<n^{-3}$ in the stationary distribution for some $\beta$ only depending on $p$. A symmetric argument shows that $\MP[j-k>\beta\log n]<n^{-3}$ for $j>k$. By union bound of the events ``element $k$ has dislocation at least $\beta\log n$'' for all $1\leq k\leq n$, we obtain the maximum dislocation claim.
\end{proof}

\subsubsection{Convergence speed (Proof of Theorem~\ref{thm:adj-speed})}
Recall Theorem~\ref{thm:adj-speed}:

\TheoremTwo*

\begin{proof} 
	Benjamini et al.~\cite{benjamini2005} show that for any constants $p<1/2$ and $\varepsilon>0$, the mixing time of the Markov chain $S_{p,1,n}$ is $T_{\textrm{mix}}(\varepsilon)=\MO(n^2)$. After mixing, the relative errors of probabilities of the resulting permutations compared to the stationary distribution are below $\epsilon$, and so are the relative errors for the distributions of the marginals $I$ and $W$. This gives us $(1-\epsilon)\E(I(S^{\;\infty}_{p,1,n}))<\E(I(x^{(t)}))<(1+\epsilon)\E(I(S^{\;\infty}_{p,1,n}))$ for any $t\geq T_{\textrm{mix}}(\epsilon)$, and similarly for $W$.
	
	On the other hand, every swap of adjacent elements can reduce the number of inversions by at most one. So in expectation, $\Omega(n^2)$ swaps are needed to go from a random permutation (with $\Theta(n^2)$ expected inversions, see Lemma~\ref{lem:upperbounds-fitness}) to a permutation with $\MO(n)$ inversions. By Theorem~\ref{thm:adj-quality}, a permutation in the stationary distribution has $\E(I)=\MO(n)$ and $\E(W)=\MO(n)$, and since by Lemma~\ref{lem:IandW} we have $W(\pi)\geq I(\pi)$, the lower bounds follow.
\end{proof}

\subsection{Sorting with any swaps}\label{sec:any}

We now analyze sorting with arbitrary swaps, i.e., we consider the Markov chain $S_{p,n}$. Observe that the Markov chain for any such process that allows non-adjacent swaps is not reversible by the Kolmogorov reversibility criterion~\cite{kelly1979reversibility}.

\begin{lemma}\label{lem:nonrev}
	The Markov chain $S_{p,n}$ over the state space $\mathcal{S}_n$ with $n>2$ is not reversible.
\end{lemma}
\begin{proof}
	Let $\pi=(1,2,3,\dots)$ be the sorted permutation, $\pi'=(2,1,3,\dots)$ be the permutation obtained from $s$ by the swap of positions 1 and 2, $\pi''=(2,3,1,\dots)$ the permutation obtained from $s'$ by the swap of positions 2 and 3, and $\pi'''=(3,2,1,\dots)$ the permutation obtained from $\pi''$ by the swap of positions 1 and 2 again. Observe that one can also obtain $\pi'''$ from $s$ by the swap of positions 1 and 3. Consider now the cycle of transitions $\mathcal{C} = \left((\pi,\pi'),(\pi',\pi''),(\pi'',\pi'''), (\pi''',\pi)\right)$ as in Section~\ref{sec:markov}: \begin{align*}
	\prod_{(i,j)\in\mathcal{C}} P(i,j) =  p^3(1-p)
	&& \text{and} &&\prod_{(i,j)\in\mathcal{C}} P(j,i) =  (1-p)^3p\, ,
	\end{align*}
	which implies by the {Kolmogorov reversibility criterion} (recalled in Section~\ref{sec:markov}) that the chain is not reversible.
\end{proof}

Finally, we prove Theorem~\ref{thm:any-exp-I-W}:

\TheoremThree*

\begin{proof}
	For the lower bounds, consider the following equivalent way to describe one step in $S_{p,n}$: With probability $(1-2p)$, arrange the two randomly chosen elements correctly, and with probability $2p$, arrange them randomly (swap or do not swap them, each with probability $p$).
	
	Let permutation $\pi$ be some state after a sufficiently large number of steps of $S_{p,n,n}$. In particular, we assume that every element has been compared at least once.
	Let $R^+$ be the set of all elements in $\pi$ whose last step during the process was a random swap.
	Let $R$ be a subset of $R^+$as follows. If the last step for two elements in $R^+$ was the same random swap, we let only one of them be in $R$ and decide which one at random. Otherwise, flip a fair coin to decide whether to include an element in $R$. 
	Clearly, for each element, the probability to be in $R$ is $\frac{p}{2}$, ($p$ to be in $R^+$ and $\frac{1}{2}$ to be in $R$ if being in $R^+$), and $\E(\lvert R \rvert) = \frac{p}{2}n$. 
	
	By construction of $R$, each element of $R$ got placed to a uniformly random position in $\pi$.
	Thus, the expected dislocation of an element in $R$ is
	\[\frac{1}{n}\cdot\frac{1}{n}\cdot2\cdot\sum_{i=1}^{n}\binom{i}{2} =
	\frac{2}{n^2}\cdot \binom{n+1}{3} =	
	 \frac{n}{3}-\frac{1}{3n}\ ,\]
	where the identity of the sum follows by \cite{concretemath} (Chapter 5.1, Equation 5.10, i.e., for $m,n\ge 0 \colon \sum_{i=0}^n\binom{i}{m} = \binom{n+1}{m+1}$).
	Therefore, by linearity of expectation, the expected total dislocation is at least $$\E(D(S_{p,n,n}^{\;\infty})) \geq\E( \lvert R \rvert)\left(\frac{n}{3}-\frac{1}{3n}\right) = \frac{p(n^2-1)}{6}\, .$$
	
	By Lemma \ref{lem:diaconis}, it holds that $I\geq \frac{1}{2}D$, thus $\E(I(S_{p,n,n}^{\;\infty})) \geq \frac{p(n^2-1)}{12}$.

	The order in which the elements of $R$ appear in $\pi$ is random.
	Moreover, the ranks of the elements in $R$ are uniformly distributed between 1 and $n$, as are their positions in which they appear in $\pi$.
	Consider the first $\frac{1}{3}n$ positions in $\pi$: The expected number of elements from $R$ that are larger than $\frac{2}{3}n$ and appear in one of these positions is $\frac{1}{9}\lvert R\rvert$. 
	By the pigeon hole principle, the number of elements that are smaller than $\frac{1}{2}n$ and appear in the middle or last third of $\pi$ is at least $\frac{1}{6}n$. All elements of this set are inverse to all elements in the first set and differ by at least $\frac{1}{6}n$.
	Therefore,
	$\E(W(S_{p,n,n}^{\;\infty})) \geq \frac{1}{6}n\cdot \frac{1}{9}\lvert R\rvert\cdot \frac{1}{6}n  \geq \frac{1}{648}pn^3 .$ 
 
	\medskip
	For the upper bounds we will use that for any permutation $\pi\in\mathcal{S}_n$,	
	\begin{align}
	W(\pi) &= \sum_{i<j \colon \pi(i) > \pi(j)} (j-i)\label{eq:W1}\\
	 &=\footnotemark \sum_{i<j \colon \pi(i) > \pi(j)} (\pi(i)-\pi(j))\, .\label{eq:W2}
	\end{align}
	\footnotetext{By Lemma \ref{lem:equivalence}, $	W(\pi) = \sum_{i\in\{1,\dots,n\}}i(i-\pi(i))$. Furthermore, since $\pi$ is a permutation of the numbers $\{1,\dots,n\}$, $\sum_{i\in\{1,\dots,n\}}i^2 = \sum_{i\in\{1,\dots,n\}}\pi(i)^2$.}
	
	\noindent
	Furthermore, we consider $\Delta^+(\pi)$  and $\Delta^-(\pi)$ as the absolute \textit{expected increase} and absolute \textit{expected decrease} of $W(\pi)$ by the next $\Swap$, respectively:
	\begin{align*}
	\Delta^+(\pi) &:= \frac{p}{{n\choose 2}}\sum_{\substack{i<j \\ \pi(i) <  \pi(j)}} (j-i)(\pi(j)-\pi(i))\, , \\
	\Delta^-(\pi) &:= \frac{1-p}{{n\choose 2}}\sum_{\substack{i<j \\ \pi(i) > \pi(j)}} (j-i)(\pi(i)-\pi(j))\, .
	\end{align*}
	To see why these formulas hold, observe that the swap of two elements $i$ and $j$ also removes or adds inversions between $i$ or $j$ and elements $k$ with $\pi(k)$ between $\pi(i)$ and $\pi(j)$, while all other inversions remain unchanged.	
	For the latter elements, the weights of their inversions change by exactly $j-i$, and there are exactly $\vert \pi(i) - \pi(j) \rvert -1$ such elements.
	
	\smallskip
	We now analyse $\Delta^+(\pi)$ and $\Delta^-(\pi)$ in the stationary distribution of the process. In the stationary distribution the expected change in $W(\pi)$ is $0 = \Delta^+(\pi) - \Delta^-(\pi)$. 
	There is an easy upper bound on $\Delta^+(\pi) \leq \MO(pn^2)$ since every summand is in $\MO(n^2)$. We will now show a lower bound for $\Delta^-(\pi)$. 
	
	Observe that since every summand in Equation \eqref{eq:W1} is smaller than $n$, there must be at least $\frac{W(\pi)}{n}$ summands, i.e., inversions. 
	Also observe that the number of pairs $i<j$ such that $j-i = x$, is at most $n$, for any $x\geq 1$. 
	Therefore, Equation \eqref{eq:W1} contains at most $n$ summands of the same value.
	%
	Now assume we do not count the smallest $\frac{W(\pi)}{2n}$ summands and let $R$ be the set of inverse pairs $(i,j)$ that correspond to the remaining summands. 
	By this greedy argument, the value of every remaining summand is at least $\frac{W(\pi)}{2n^2}$. 
	Moreover, both the sum in Equation \eqref{eq:W1} and the sum in Equation \eqref{eq:W2} are at least $\frac{W(\pi)}{2}$, since we ignore at most a total value of $\frac{W(\pi)}{2n} \cdot n$.
	Therefore, we get 
	\begin{align*}
	\Delta^-(\pi) &\geq \frac{1-p}{{n\choose 2}} \sum_{(i,j)\in R} (j-i)(\pi(i)-\pi(j))\\
	&\ge \frac{1-p}{{n\choose 2}}\cdot \frac{W(\pi)}{2n^2} \sum_{(i,j)\in R} (\pi(i)-\pi(j)) \\
	&\geq \frac{1-p}{{n\choose 2}}\cdot \frac{W(\pi)}{2n^2}\cdot\frac{W(\pi)}{2}\\
	&= \Omega\left(\frac{W(\pi)^2}{n^4}\right)\, .
	\end{align*}
	Notice that we can ignore the factor $1-p$, since by the assumption of the theorem, the factor is larger than $\frac{1}{2}$ and smaller than 1.
	
	By Jensen's inequality and linearity of expectation we get that
	\begin{align*}
	\E(\Delta^-(S_{p,n,n}^{\;\infty})) \geq \Omega\left(\E\left(\frac{W(S_{p,n,n}^{\;\infty})^2}{n^4}\right)\right) \geq \Omega\left(\frac{\E(W(S_{p,n,n}^{\;\infty}))^2}{n^4}\right)\, .
	\end{align*}
	
	If we combine this lower bound with the upper bound for $\Delta^+(\pi)$ we can conclude 
	\begin{align*}
	\Omega\left(\frac{\E(W(S_{p,n,n}^{\;\infty}))^2}{n^4}\right)\leq \Delta^-(S_{p,n,n}^{\;\infty}) = \Delta^+(S_{p,n,n}^{\;\infty}) \leq \MO(pn^2)\, ,
	\end{align*}
	which implies that $\E(W(S_{p,n,n}^{\;\infty}))\leq \MO(\sqrt{p}\ n^3)$.
	
	We proceed in a similar way to derive the upper bound for $D(\pi)$ and $I(\pi)$. This time we ignore the $\frac{I(\pi)}{3}$ smallest summands in Equation \eqref{eq:W1} and the $\frac{I(\pi)}{3}$ smallest summands in Equation \eqref{eq:W2}. By pigeon hole principle, we remain with at least $\frac{I(\pi)}{3}$ summands for $\Delta^-(\pi)$ that each have a value of at least $\frac{I(\pi)^2}{9n^2}$. Therefore, 
	$$\Delta^-(\pi) \geq \Omega\left(\frac{I(\pi)^3}{n^4}\right).$$
	For $\Delta^+(\pi)$ we get the analogue upper bound of $\MO(p n^2)$. Finally, for a stationary state, 
	\begin{align*}
	\Omega\left(\frac{\E(I(S_{p,n,n}^{\;\infty}))^3}{n^4}\right) \leq \Delta^-(S_{p,n,n}^{\;\infty}) = \Delta^+(S_{p,n,n}^{\;\infty}) \leq \MO(p n^2)\, ,
	\end{align*}
	which implies that $\E(I(S_{p,n,n}^{\;\infty})) \leq \MO(p^{1/3}n^2)$. The upper bound for $D(S_{p,n,n}^{\;\infty})$ follows by Lemma \ref{lem:diaconis}.
\end{proof}

\subsection{Sorting with bounded-range swaps}\label{sec:bounded}
Finally, we turn to the general process, where we allow swaps between elements that lie at most $r$ positions apart, i.e., we consider the Markov chain of $S_{p,r}$.
Observe that this chain is not reversible, since for $r\ge 2$ the same example as in the proof of non-reversibility of $S_{p,n}$ (see Lemma~\ref{lem:nonrev}) applies.

Recall Theorem~\ref{thm:range-swap}:

\TheoremFour*

\begin{proof} 
	The proof is similar to the one of Theorem~\ref{thm:any-exp-I-W}: We consider again the equivalent process description and the sets $R^+$ (the set of all elements whose last step was a random swap) and $R$ (the subset of $R^+$ that includes each element with probability $\frac{1}{2}$), with $\E(\lvert R\rvert)=\frac{pn}{2}$. 
	
	For each element in $R$, its new expected dislocation will be larger than $\Omega(r)$, since the element is placed to a random position inside a radius $r$ compared to its old position.
	Therefore, $\E(D(\pi)) \geq \lvert R\rvert\cdot\Omega(r) =  \Omega(prn)$, and by \cite{diaconis1977}, also $\E(I(\pi)) \geq  \Omega(prn)$.
	
	For the weighted number of inversions, we observe the following: If an element $i$ has dislocation $d_i$, then it is inverse to at least $d_i$ pairwise different larger or smaller elements (no such element can have the same difference to $i$), and its contribution to $W$ is at least $\frac{1}{2}{d_i+1\choose 2}$, where the factor $\tfrac12$ is to prevent double counting.
	The expected value of ${d_i+1\choose 2}$ is in $\Omega(r^2)$. Thus, $\E(W(\pi)) \geq \lvert R\rvert\cdot\Omega(r^2) =  \Omega(pr^2n)$.
\end{proof}

\section{Experimental results}\label{sec:experimental}

We complement the theoretical analysis of Section~\ref{sec:theory} with experimental observations.

\subsection{Methodology}\label{sec:methodology}

The experimental results were obtained by simulating the process for a given number of steps or until the fitness function converged (for the convergence criterion see Section~\ref{sec:conv-time}). The simulation is implemented in a combination of C++ and Python, the sources are freely available on GitHub\footnote{\url{https://github.com/gavento/swap-sorting-experiments}}. For reproducibility, the simulation uses a pseudo-random generator with a deterministic seed.

Every plot is based on 300 independent runs. We generally use $n$ around 512, as the experimental results are consistent already for $n\geq 128$. All the plots include 95\,\% confidence interval error bars even where these are too small to be visible.

To illustrate the evolution of the solution quality over time, see Figure~\ref{fig:process}.

\begin{figure}[t]\begin{center}
		\includegraphics[width=0.49\columnwidth]{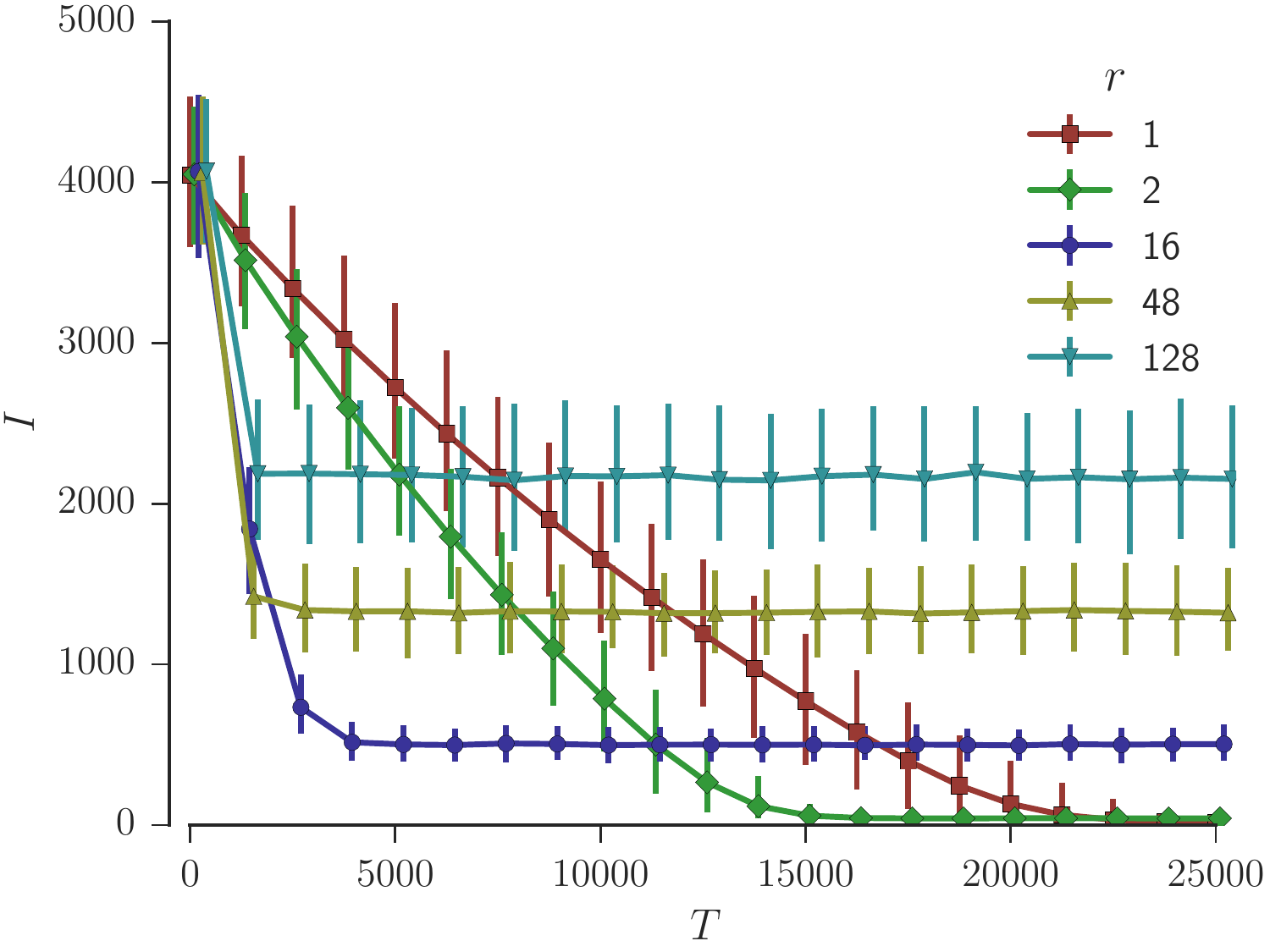}
		\includegraphics[width=0.49\columnwidth]{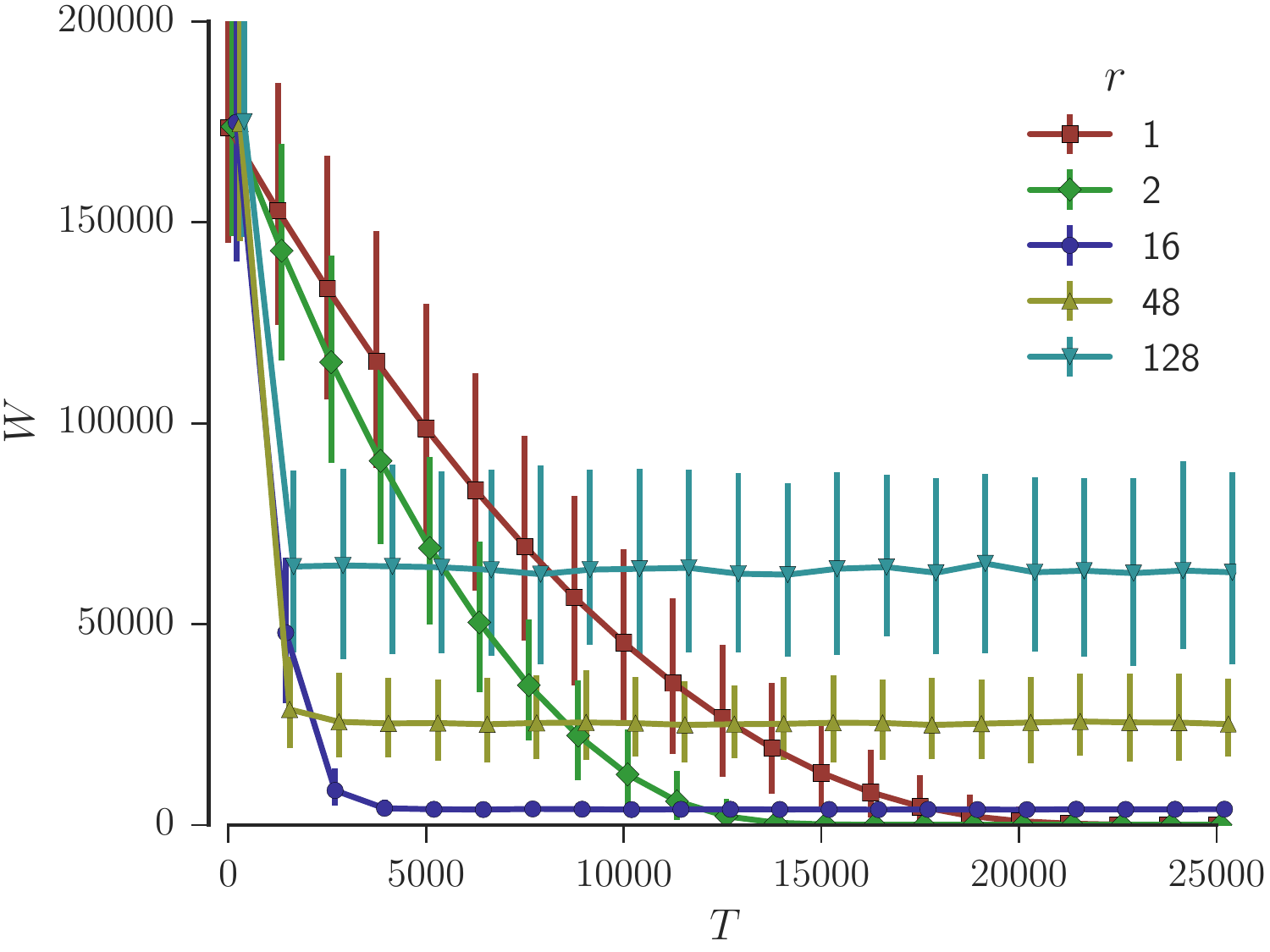}
		\caption{Illustration of development and stabilization of $I$ (left) and $W$ (right) over time for different values of $r$ (swap distance) with $n=512$ and $p=0.1$.
			Mean and 95\,\% confidence intervals over 300 runs.
		}
		\label{fig:process}
		\vspace{-4mm}
\end{center}\end{figure}

\begin{figure}[h]\begin{center}
		\includegraphics[width=0.7\columnwidth]{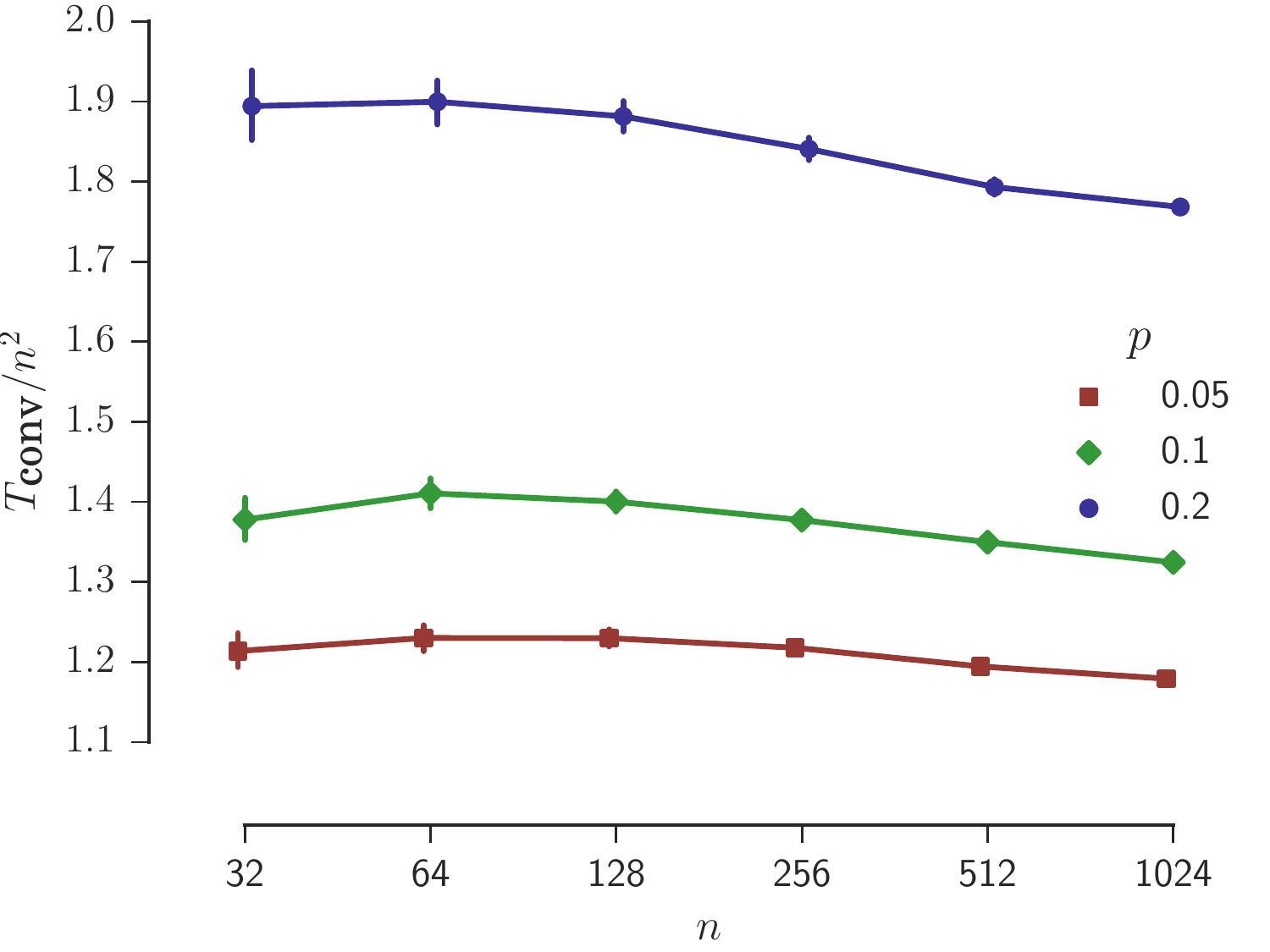}
		\vskip 3mm
		\includegraphics[width=0.7\columnwidth]{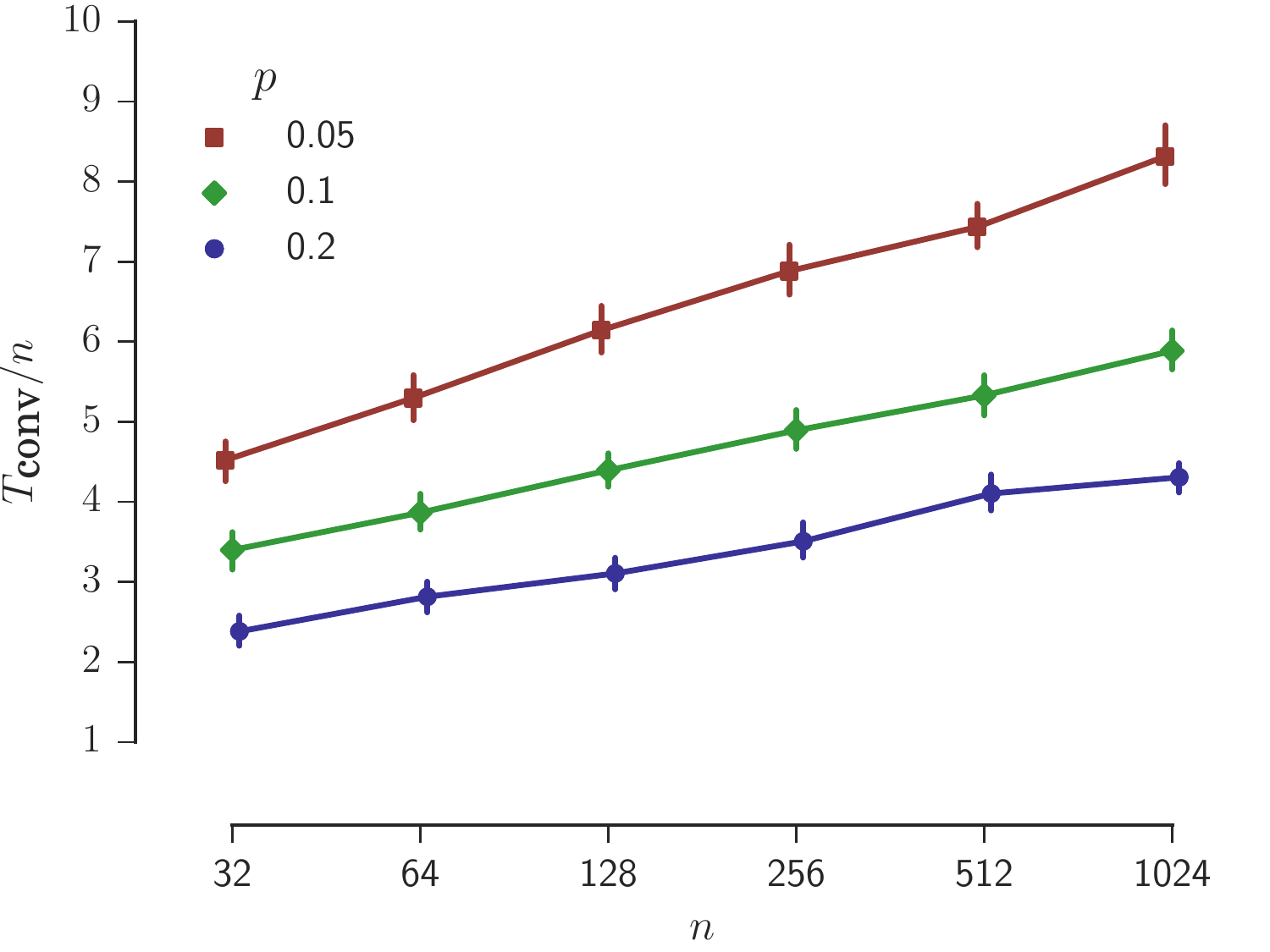}
		\caption{\emph{Top:} $\Tconv$ normalized by $n^2$ for $r=1$. \emph{Bottom:} $\Tconv$ normalized by $n$ for $r=n$. \emph{Both:} mean and $95\,\%$ confidence intervals over 300 runs for each choice of $n$ and $p$.}
		\label{fig:T-by-n}
\end{center}\end{figure}

\begin{figure}[th]\begin{center}
		\includegraphics[width=0.8\columnwidth]{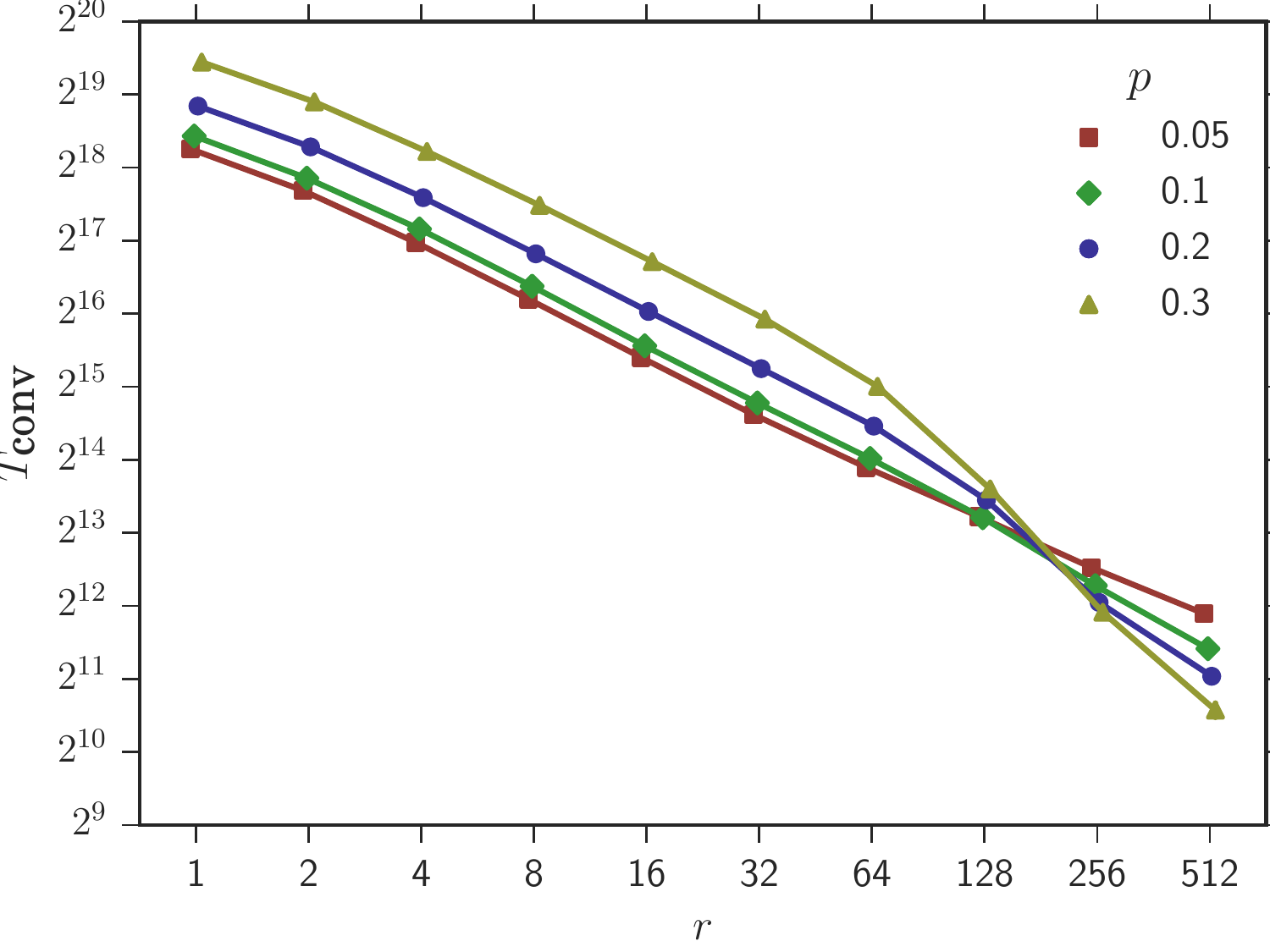}
		\caption{$\Tconv$ for $n=512$ and various $p$ and $r$ (mean and 95\,\% confidence intervals over 300 runs).}
		\label{fig:T-by-r}
\end{center}\end{figure}

\subsection{Convergence time criterion}\label{sec:conv-time}

In all the experiments, we use the convergence of $I(\pi^{(t)})$ as the main stopping criterion\footnote{Experiments show that the convergence behavior of $W(\pi)$ is very similar and the convergence would differ by less than $5\%$.}. Based on experiments, we assume the sequence $I(\pi^{(t)})$ for $t=0,1,\dots$ is a monotonically decreasing function with an additive noise term. We define the convergence time $\Tconv$ as the smallest time $t$ when for a random starting permutation $\pi\in\mathcal{S}_n$, \[\E(I(\pi^{(t)}))
\leq(1+\epsilon)\E(I(S^{\;\infty}_{p,r,n}))\, .\]
For an overview and discussion on various stopping criteria, see~\cite{Safe2004}. 

To estimate $\E(I(\pi^{(t)}))$, we average over a sliding window starting at time $t$. To choose an appropriate size of the window proportional to the convergence time, we estimate the convergence time as
\[\Tcest=\frac{n^2}{\overline{r}(1-2p)}\ ,\]
where $\bar{r}$ is the average swap length, e.g., $\overline{r}=(\sum_{i=1}^r i(n-i))/(\sum_{i=1}^r (n-i))$. The experiments show that this estimate is within $0.2$ to $2.3$ multiplicative error of the measured mean $\Tconv$ on our data. Note that the asymptotic behavior may be different and we need only a rough estimate.

We choose $\epsilon=0.05$ and window size $w=\lceil0.05 \Tcest\rceil=\Omega(n)$. We also set a sampling rate $s=\lceil \Tcest/1000\rceil$ to speed up the computations.

In step $t$, such that $t$ is divisible by $s$ and $t>3w$, we compute the mean $I$ of windows starting at $t$ and $\rfrac{3}{2}t$:
\begin{align*}
\overline{I(\pi^{(t)})}= \frac{1}{w}\sum_{i=t}^{t+w-1}I(\pi^{(i)}) \quad \text{and} \quad
\overline{I(\pi^{(\rfrac{3}{2}t)})}=\frac{2}{t}\sum_{i=\rfrac{3}{2}t+1}^{2t}I(\pi^{(i)})\ .
\end{align*}
If $\overline{I(\pi^{(t)})}\leq(1+\epsilon)\overline{I(\pi^{(\rfrac{3}{2}t)})}$, then we estimate $\Tconv=t$.

If $t\geq \rfrac{2}{3}\cdot\Tconv$, then $\overline{I(\pi^{(\rfrac{3}{2}t)})}$ is an estimate for $\E(I(S^{\;\infty}_{p,r,n}))$. On the other hand if $t<\rfrac{2}{3}\cdot\Tconv$, then $t$ and $\rfrac{3}{2}t$ are both in the not-converged phase. Therefore, only a very small average descent of $I$ between the windows would imply to a wrong estimate.

\subsection{Convergence time results}

See Figure~\ref{fig:T-by-n} for the plots of mean $\Tconv$ for the extreme cases $r=1$ and $r=n$, and Figure~\ref{fig:T-by-r} for a dependency on $r$.

The experimental results indicate that for a fixed $p$ and $r=1$, $\Tconv=\Theta(n^2)$ (in accordance with Theorem~\ref{thm:adj-speed}). The results for $r=n$ are less conclusive but might suggest $\Tconv\sim n\log n$.

Note that for $r$ close to $n$, the time measurements might be less accurate due to large fitness changes and the fitness of the stationary distribution being relatively close to that of a random permutation (see Figure~\ref{fig:IW-by-p-rn}). However, any imprecision of convergence time measurement should have no effect on the converged fitness measurements. 

\begin{figure}[t]\begin{center}
		\includegraphics[width=0.48\columnwidth]{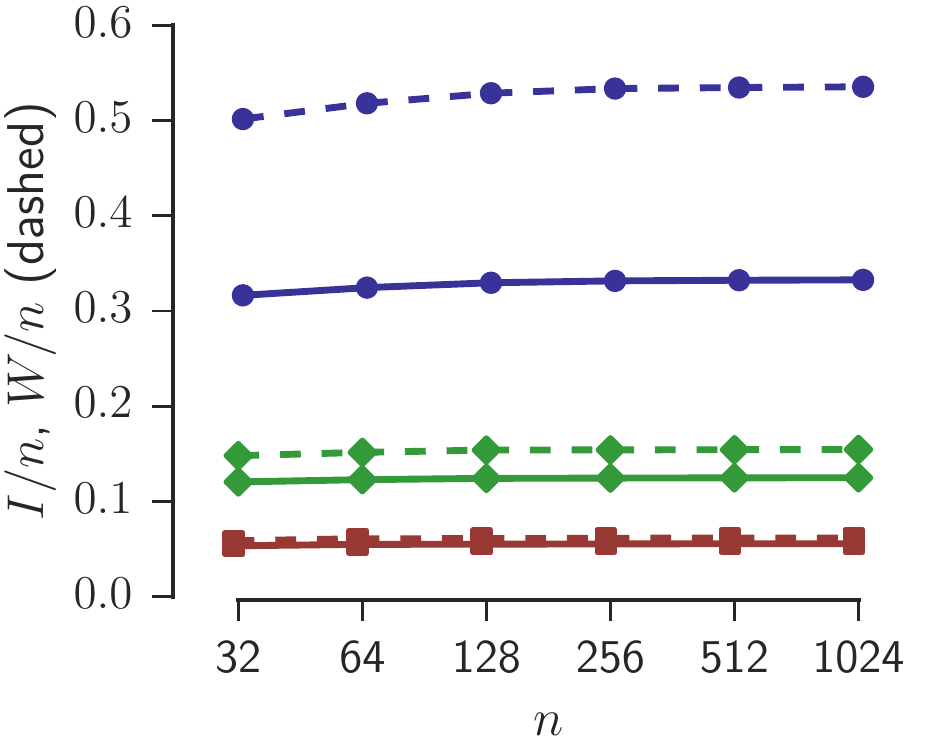}
		\includegraphics[width=0.48\columnwidth]{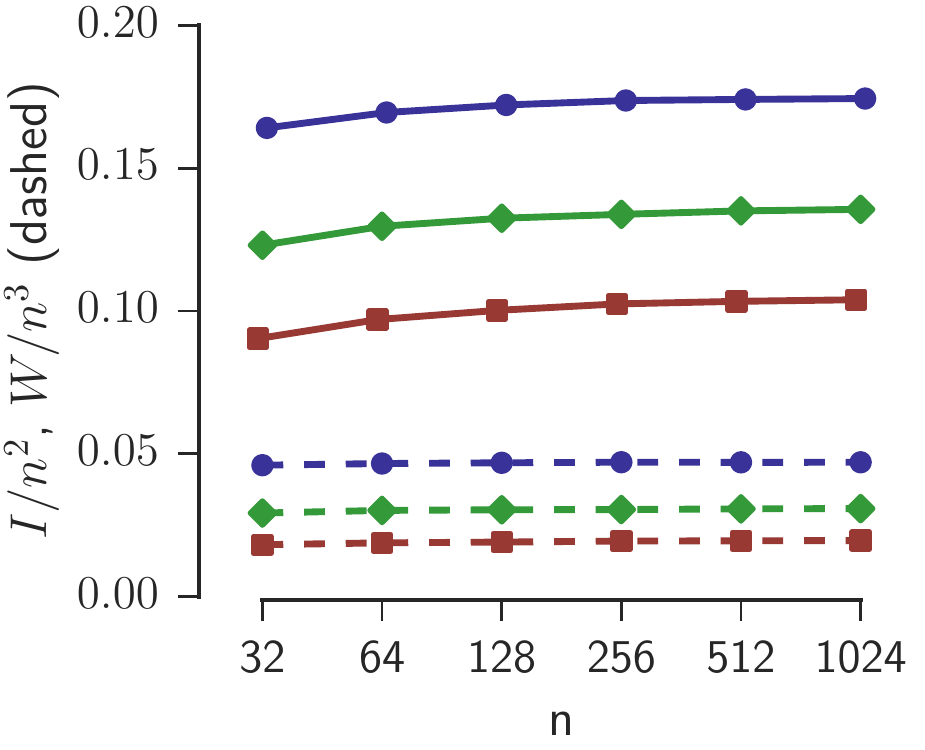}
		\includegraphics[width=0.5\columnwidth]{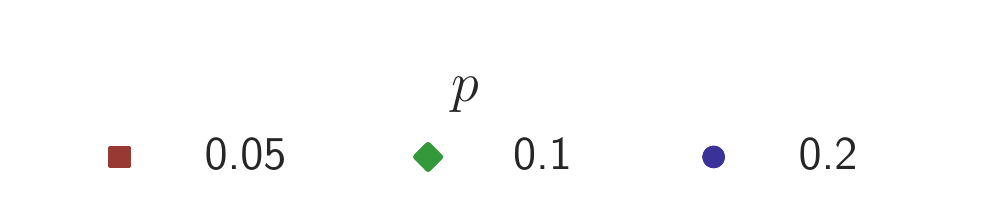}
		\caption{Stability of converged $I/n$ and $W/n$ with increasing $n$ for various $p$.
		\emph{Left:} $I/n$ and $W/n$ for $r=1$.
		\emph{Right:} $I/n^2$ and $W/n^3$ for $r=n$.
		\emph{In both:} Converged phase means and 95\,\% confidence intervals (too small to see) over 300 runs.}
		\label{fig:IW-by-n}
		\vspace{-3mm}
\end{center}\end{figure}

\begin{figure}[thb]\begin{center}
		\includegraphics[width=0.8\columnwidth]{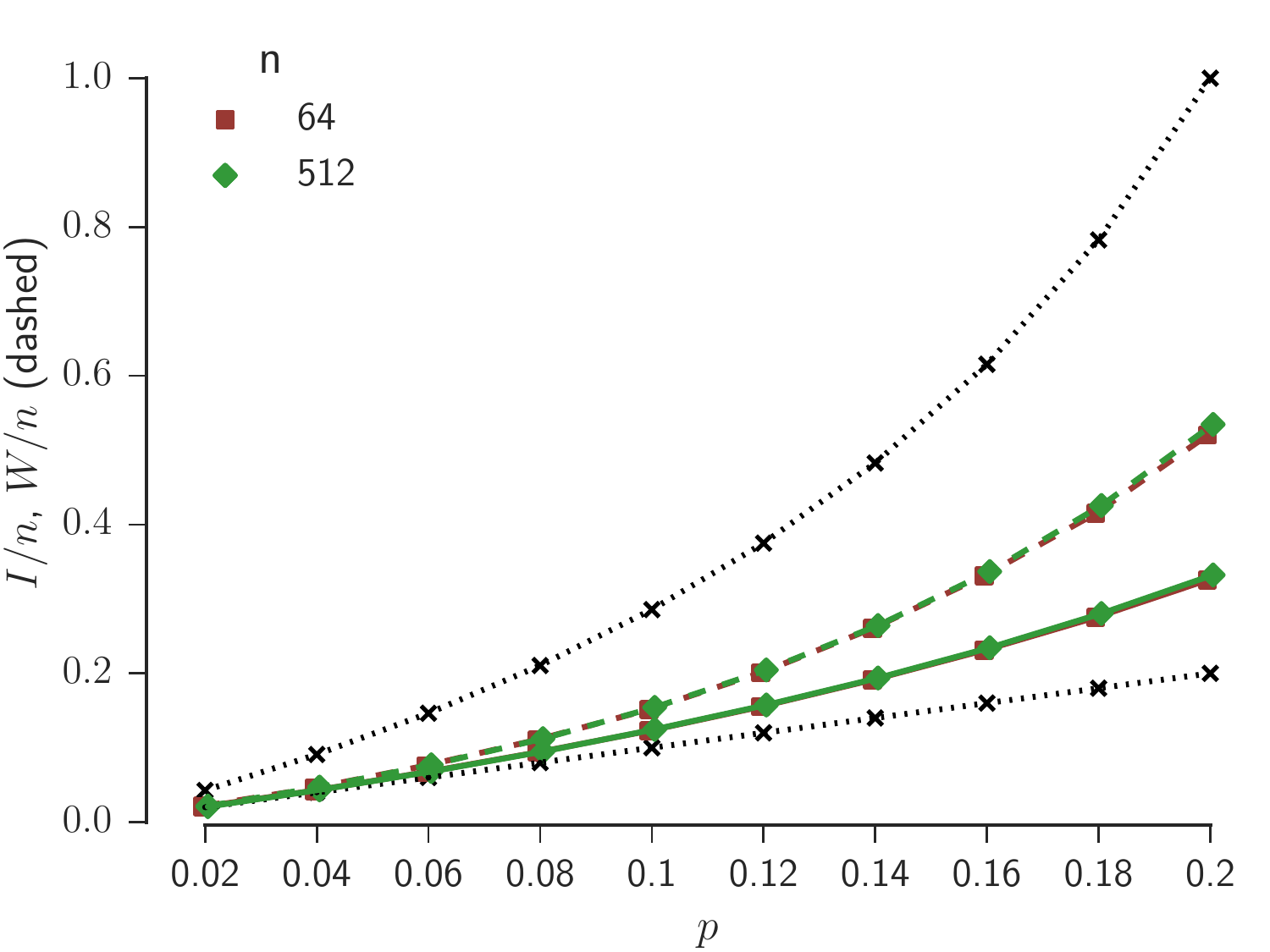}
		\vskip 3mm
		\includegraphics[width=0.8\columnwidth]{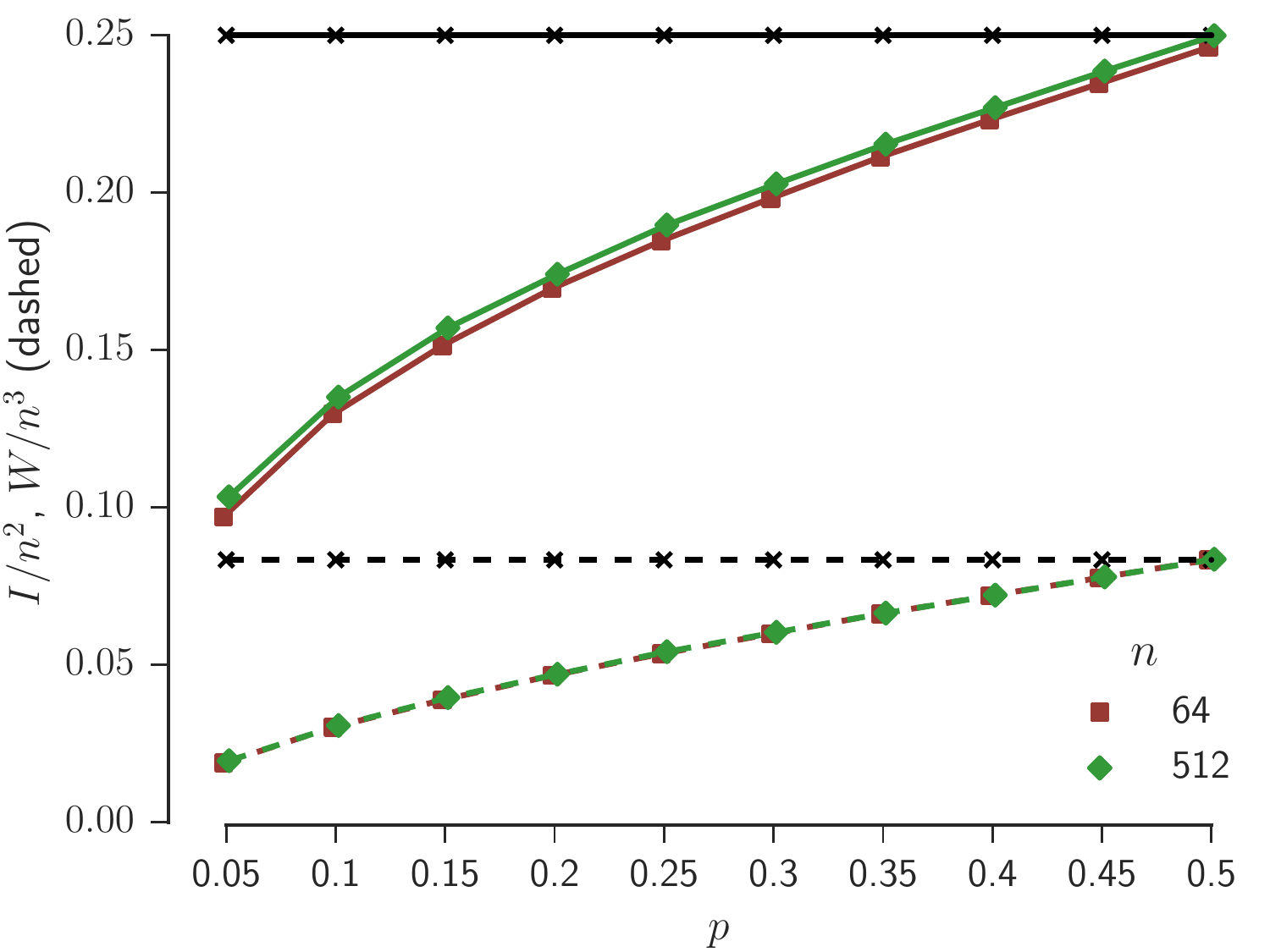}
		\caption{
		\emph{Top:} Converged $I/n$ (solid) and $W/n$ (dashed) for $r=1$ and various $n$ and $p$. Note that $n=64$ and $n=512$ overlap. The dotted lines are the upper and lower bounds on both $\E(I/n)$ and $\E(W/n)$ from Theorem~\ref{thm:adj-quality}.
		\emph{Bottom:} Converged $I/n^2$ (solid) and $W/n^3$ (dashed) for $r=n$ and various $n$ and $p$. The horizontal lines indicate the expectation of a random permutation.
		\emph{Both:} Converged phase mean and 95\,\% confidence intervals over 300 runs.
		}
		\label{fig:IW-by-p-r1}
		\label{fig:IW-by-p-rn}
		\vspace{-3mm}
\end{center}\end{figure}

\subsection{Converged state quality}\label{sec:conv-quality}

We estimate the qualitative properties $I$ and $W$ of the stationary distribution as the distribution of the values in the range $[{3}/{2}\cdot\Tconv,$ $2\cdot\Tconv]$ over 300 independent process runs. In all the data, note that the 95\,\% confidence error bars are very small and generally not visible.

For values of $p\leq 0.3$ and $n\leq 1024$, we observe that estimates $I(S^{\;\infty}_{p,1,n})\simeq f_1(p)\cdot{n}$, $W(S^{\;\infty}_{p,1,n})\simeq f_2(p)\cdot n$, $I(S^{\;\infty}_{p,n,n})\simeq f_3(p)\cdot n^2$, and $W(S^{\;\infty}_{p,n,n})\simeq f_4(p)\cdot n^3$ (for some fixed unspecified functions $f_1,\dots, f_4$) are surprisingly accurate. See Figure~\ref{fig:IW-by-n}.

To estimate the dependency on $p$ (e.g., the functions $f_1,\dots f_4$), see Figure~\ref{fig:IW-by-p-r1}. Finally, see Figure~\ref{fig:I-by-r} for the experimental dependency of $I$ and $W$ on $r$. This trade-off corresponds to the lower bounds of Theorem~\ref{thm:range-swap}. Note the non-linearity at $r\geq 128$ is likely to be caused by the average swap-distance $\bar{r}$ being lower than $r/2$ (while $\bar{r}\simeq r/2$ when $r\ll n$).

\begin{figure}[tb]\begin{center}
		\includegraphics[width=0.8\columnwidth]{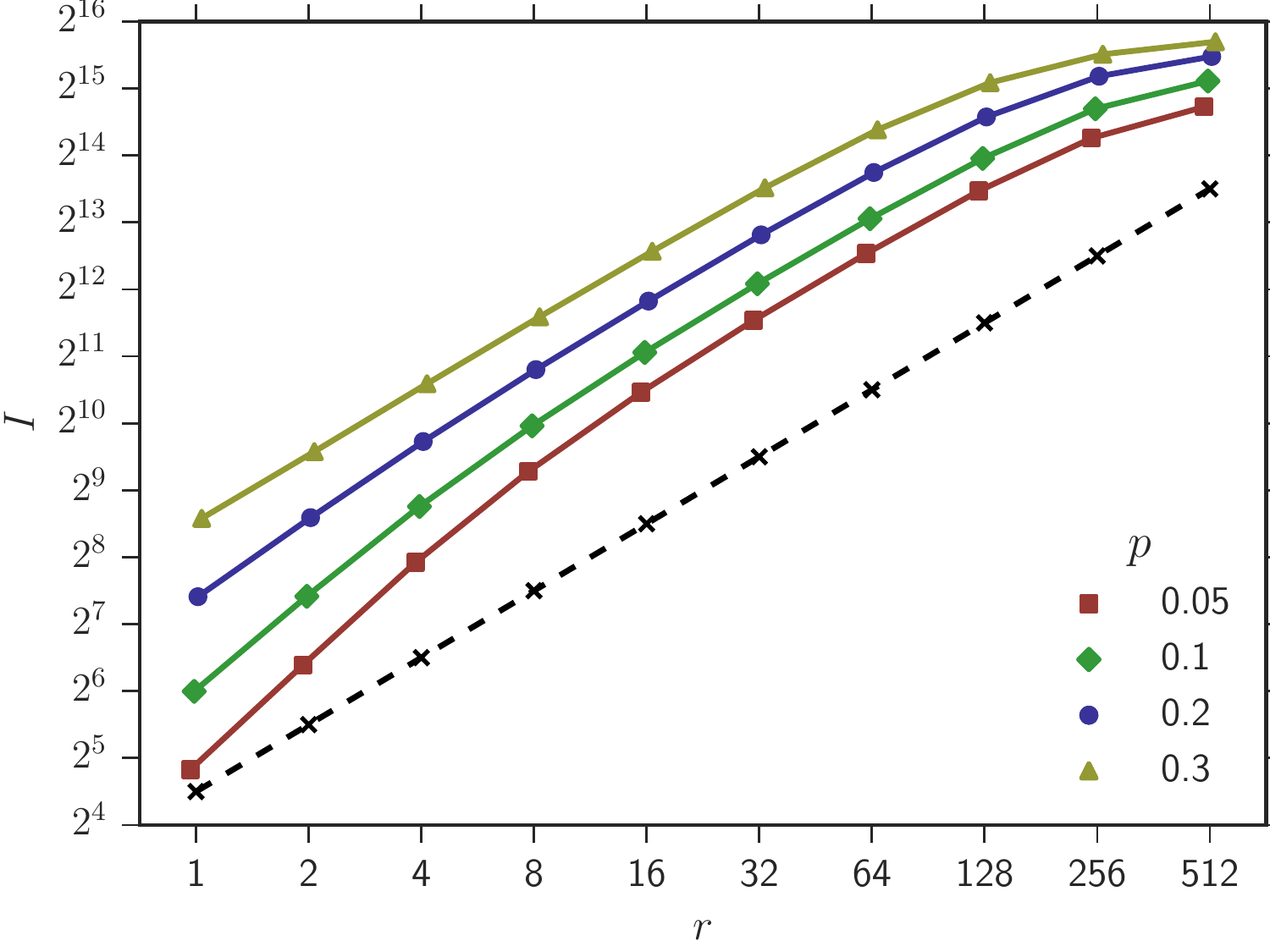}
		\vskip 3mm
		\includegraphics[width=0.8\columnwidth]{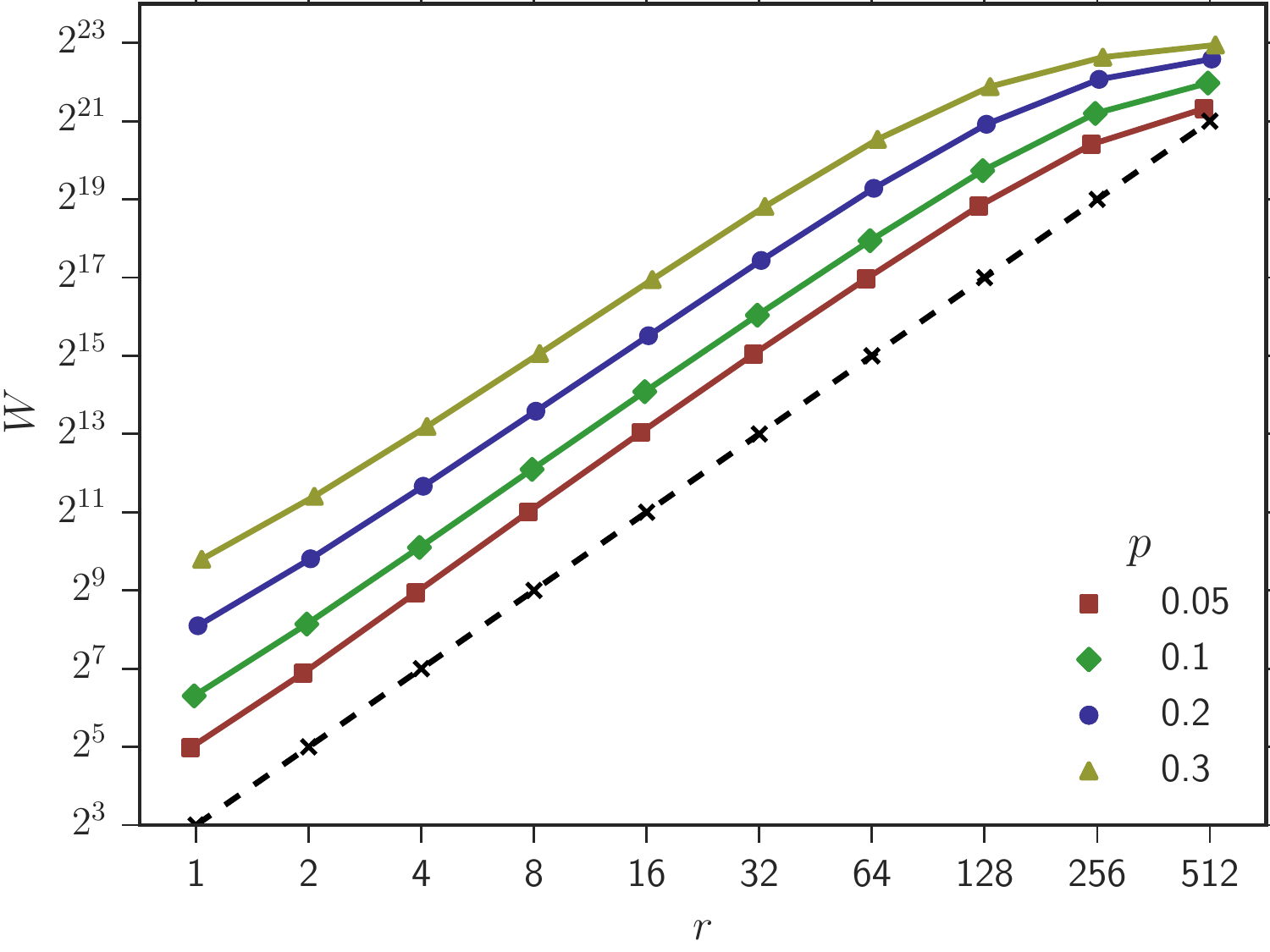}
		\caption{Dependence of the converged $I$ (top) and $W$ (bottom) on $r$ for $n=512$ with various $p$ (converged phase mean and 95\,\% confidence intervals over 300 runs). The dashed line is the \emph{direction} of the lower-bound from Theorem~\ref{thm:range-swap} (ignoring a multiplicative factor depending on $p$).
		}
		\label{fig:I-by-r}
		\vspace{-2mm}
\end{center}\end{figure}


\section{Conclusion}
We have studied sorting by random swaps with a noisy comparison operator. We considered swaps of elements in distance at most $r$, and we found a trade-off between fast convergence (for large $r$) and high quality of the solution (for small $r$). As most of our theoretical results are for the extreme cases $r=1$ and $r=n$, a natural next step is to verify theoretically the experimental results for arbitrary $r$, in particular, to prove upper bounds that match the lower bounds in Theorem~\ref{thm:range-swap}, and to compute the convergence time for general $r$. On the other hand, one can regard the limiting case $p = p(n) \to 0$; in this regime there is still a gap ($p$ vs. $p^{1/3}$ and $p$ vs. $p^{1/2}$) in Theorem~\ref{thm:any-exp-I-W}.

Since all parameter choices have strengths and weaknesses, an important question is whether an adaptive algorithm that decreases $r$ over time (similar to Simulated Annealing) can be strictly superior to any fixed-parameter choice: In particular, can such an algorithm achieve a linear expected (weighted) number of inversions in the stationary distribution with a sub-quadratic convergence time? 
We leave this question open for future research.

\clearpage

\bibliographystyle{abbrv}      
\bibliography{references2}   

\begin{thebibliography}{10}

\bibitem{akimoto2015analysis}
Y.~Akimoto, S.~A. Morales, and O.~Teytaud.
\newblock Analysis of runtime of optimization algorithms for noisy functions
  over discrete codomains.
\newblock {\em Theoretical Computer Science}, 605:42--50, 2015.

\bibitem{assaf1991fault}
S.~Assaf and E.~Upfal.
\newblock Fault tolerant sorting networks.
\newblock {\em {SIAM} Journal on Discrete Mathematics}, 4(4):472--480, 1991.

\bibitem{astete2016simple}
S.~Astete-Morales, M.~Cauwet, J.~Liu, and O.~Teytaud.
\newblock Simple and cumulative regret for continuous noisy optimization.
\newblock {\em Theoretical Computer Science}, 617:12--27, 2016.

\bibitem{astete2015evolution}
S.~Astete-Morales, M.~Cauwet, and O.~Teytaud.
\newblock Evolution strategies with additive noise: {A} convergence rate lower
  bound.
\newblock In J.~He, T.~Jansen, G.~Ochoa, and C.~Zarges, editors, {\em
  Proceedings of the 2015 {ACM} Conference on Foundations of Genetic Algorithms
  XIII, Aberystwyth, United Kingdom, January 17 - 20, 2015}, pages 76--84.
  {ACM}, 2015.

\bibitem{astete2016analysis}
S.~Astete-Morales, M.~Cauwet, and O.~Teytaud.
\newblock Analysis of different types of regret in continuous noisy
  optimization.
\newblock In T.~Friedrich, F.~Neumann, and A.~M. Sutton, editors, {\em
  Proceedings of the 2016 on Genetic and Evolutionary Computation Conference,
  Denver, CO, USA, July 20 - 24, 2016}, pages 205--212. {ACM}, 2016.

\bibitem{auger2011theory}
A.~Auger and B.~Doerr.
\newblock {\em Theory of randomized search heuristics: Foundations and recent
  developments}, volume~1 of {\em Series On Theoretical Computer Science}.
\newblock World Scientific Publishing Co., Inc., River Edge, NJ, USA, 2011.

\bibitem{benjamini2005}
I.~Benjamini, N.~Berger, C.~Hoffman, and E.~Mossel.
\newblock Mixing times of the biased card shuffling and the asymmetric
  exclusion process.
\newblock {\em Transactions of the American Mathematical Society},
  357(8):3013--3029, 2005.

\bibitem{beyer2000evolutionary}
H.-G. Beyer.
\newblock Evolutionary algorithms in noisy environments: Theoretical issues and
  guidelines for practice.
\newblock {\em Computer methods in applied mechanics and engineering},
  186(2):239--267, 2000.

\bibitem{cauwet2016algorithm}
M.~Cauwet, J.~Liu, B.~Rozi{\`{e}}re, and O.~Teytaud.
\newblock Algorithm portfolios for noisy optimization.
\newblock {\em Annals of Mathematics and Artificial Intelligence},
  76(1-2):143--172, 2016.

\bibitem{DangJL17}
D.~Dang, T.~Jansen, and P.~K. Lehre.
\newblock Populations can be essential in tracking dynamic optima.
\newblock {\em Algorithmica}, 78(2):660--680, 2017.

\bibitem{dasgupta2013evolutionary}
D.~Dasgupta and Z.~Michalewicz.
\newblock {\em Evolutionary algorithms in engineering applications}.
\newblock Springer-Verlag New York, Inc., Secaucus, NJ, USA, 1st edition, 1997.

\bibitem{opac-b1087294}
P.~Diaconis.
\newblock {\em Group Representations in Probability and Statistics}, volume~11.
\newblock Institute of Mathematical Statistics, 1988.

\bibitem{diaconis1977}
P.~Diaconis and R.~L. Graham.
\newblock Spearman's footrule as a measure of disarray.
\newblock {\em Journal of the Royal Statistical Society. Series B
  (Methodological)}, 39(2):262--268, 1977.

\bibitem{doerr2008directed}
B.~Doerr and E.~Happ.
\newblock Directed trees: {A} powerful representation for sorting and ordering
  problems.
\newblock In {\em Proceedings of the {IEEE} Congress on Evolutionary
  Computation, {CEC} 2008, June 1-6, 2008, Hong Kong, China}, pages 3606--3613.
  {IEEE}, 2008.

\bibitem{droste2004analysis}
S.~Droste.
\newblock Analysis of the {(1+1)} {EA} for a noisy {OneMax}.
\newblock In K.~Deb, R.~Poli, W.~Banzhaf, H.~Beyer, E.~K. Burke, P.~J. Darwen,
  D.~Dasgupta, D.~Floreano, J.~A. Foster, M.~Harman, O.~Holland, P.~L. Lanzi,
  L.~Spector, A.~Tettamanzi, D.~Thierens, and A.~M. Tyrrell, editors, {\em
  Genetic and Evolutionary Computation - {GECCO} 2004, Genetic and Evolutionary
  Computation Conference, Seattle, WA, USA, June 26-30, 2004, Proceedings, Part
  {I}}, volume 3102 of {\em Lecture Notes in Computer Science}, pages
  1088--1099. Springer, 2004.

\bibitem{friedrich2016compact}
T.~Friedrich, T.~K{\"{o}}tzing, M.~S. Krejca, and A.~M. Sutton.
\newblock The compact genetic algorithm is efficient under extreme gaussian
  noise.
\newblock {\em {IEEE} Transactions on Evolutionary Computation},
  21(3):477--490, 2017.

\bibitem{gavenvciak2017sorting}
T.~Gavenciak, B.~Geissmann, and J.~Lengler.
\newblock Sorting by swaps with noisy comparisons.
\newblock In {\em Proceedings of the Genetic and Evolutionary Computation
  Conference, {GECCO} 2017, Berlin, Germany, July 15-19, 2017}, pages
  1375--1382, 2017.

\bibitem{barbaraPaper}
B.~{Geissmann} and P.~{Penna}.
\newblock {Sort well with energy-constrained comparisons}.
\newblock {\em ArXiv e-prints}, 2016.

\bibitem{Giesen2009}
J.~Giesen, E.~Schuberth, and M.~Stojakovi\'{c}.
\newblock Approximate sorting.
\newblock {\em Fundamenta Informaticae}, 90(1-2):67--72, 2009.

\bibitem{GiessenK16}
C.~Gie{\ss}en and T.~K{\"{o}}tzing.
\newblock Robustness of populations in stochastic environments.
\newblock {\em Algorithmica}, 75(3):462--489, 2016.

\bibitem{concretemath}
R.~L. Graham, D.~E. Knuth, and O.~Patashnik.
\newblock {\em Concrete Mathematics: A Foundation for Computer Science}.
\newblock Addison-Wesley Professional, Reading, MA, USA, 2nd edition, 1994.

\bibitem{hadjicostas2015}
P.~Hadjicostas and C.~Monico.
\newblock A new inequality related to the {D}iaconis-{G}raham inequalities and
  a new characterisation of the dihedral group.
\newblock {\em Australasian Journal of Combinatorics}, 63(2):226--245, 2015.

\bibitem{jansen2013analyzing}
T.~Jansen.
\newblock {\em Analyzing Evolutionary Algorithms - The Computer Science
  Perspective}.
\newblock Natural Computing Series. Springer Science \& Business Media, 2013.

\bibitem{kelly1979reversibility}
F.~Kelly.
\newblock {\em Reversibility and stochastic networks}.
\newblock Wiley series in probability and mathematical statistics: Tracts on
  probability and statistics. J. Wiley, 1979.

\bibitem{KotzingLW15}
T.~K{\"{o}}tzing, A.~Lissovoi, and C.~Witt.
\newblock {(1+1)} {EA} on generalized dynamic {OneMax}.
\newblock In {\em Proceedings of the 2015 {ACM} Conference on Foundations of
  Genetic Algorithms XIII, Aberystwyth, United Kingdom, January 17 - 20, 2015},
  pages 40--51, 2015.

\bibitem{LevPerWil09}
D.~Levin and Y.~Peres.
\newblock {\em {M}arkov chains and mixing times: Second Edition}.
\newblock MBK. American Mathematical Society, 2017.

\bibitem{liu2014mathematically}
J.~Liu, D.~L. St-Pierre, and O.~Teytaud.
\newblock A mathematically derived number of resamplings for noisy
  optimization.
\newblock In {\em Proceedings of the Companion Publication of the 2014 Annual
  Conference on Genetic and Evolutionary Computation}, GECCO Comp '14, pages
  61--62, New York, NY, USA, 2014. ACM.

\bibitem{ma1994fault}
Y.~Ma.
\newblock {\em Fault-tolerant sorting networks}.
\newblock PhD thesis, Massachusetts Institute of Technology, 1994.

\bibitem{merelo2016statistical}
J.~Merelo, Z.~Chelly, A.~Mora, A.~Fern{\'a}ndez-Ares, A.~I.
  Esparcia-Alc{\'a}zar, C.~Cotta, P.~de~las Cuevas, and N.~Rico.
\newblock A statistical approach to dealing with noisy fitness in evolutionary
  algorithms.
\newblock In {\em Computational Intelligence: International Joint Conference,
  IJCCI 2014 Rome, Italy, October 22-24, 2014 Revised Selected Papers}, pages
  79--95. Springer International Publishing, 2016.

\bibitem{Mitchell04}
L.~H. Mitchell.
\newblock Maximal total absolute displacement of a permutation.
\newblock {\em Discrete Mathematics}, 274(1-3):319--321, 2004.

\bibitem{neumann2013bioinspired}
F.~Neumann and C.~Witt.
\newblock {\em Bioinspired computation in combinatorial optimization:
  algorithms and their computational complexity}.
\newblock Springer-Verlag New York, Inc., New York, NY, USA, 1st edition, 2010.

\bibitem{nix1992modeling}
A.~E. Nix and M.~D. Vose.
\newblock Modeling genetic algorithms with markov chains.
\newblock {\em Annals of Mathematics and Artificial Intelligence}, 5(1):77--88,
  1992.

\bibitem{qian2016effectiveness}
C.~Qian, Y.~Yu, Y.~Jin, and Z.~Zhou.
\newblock On the effectiveness of sampling for evolutionary optimization in
  noisy environments.
\newblock 8672:302--311, 2014.

\bibitem{qian2015analyzing}
C.~Qian, Y.~Yu, and Z.-H. Zhou.
\newblock Analyzing evolutionary optimization in noisy environments.
\newblock {\em Evolutionary computation}, 2015.

\bibitem{rakshit2016noisy}
P.~Rakshit, A.~Konar, and S.~Das.
\newblock Noisy evolutionary optimization algorithms - {A} comprehensive
  survey.
\newblock {\em Swarm and Evolutionary Computation}, 33:18--45, 2017.

\bibitem{Safe2004}
M.~D. Safe, J.~A. Carballido, I.~Ponzoni, and N.~B. Brignole.
\newblock On stopping criteria for genetic algorithms.
\newblock In {\em Advances in Artificial Intelligence - {SBIA} 2004, 17th
  Brazilian Symposium on Artificial Intelligence, S{\~{a}}o Luis,
  Maranh{\~{a}}o, Brazil, September 29 - October 1, 2004, Proceedings}, pages
  405--413, 2004.

\bibitem{scharnow2004analysis}
J.~Scharnow, K.~Tinnefeld, and I.~Wegener.
\newblock The analysis of evolutionary algorithms on sorting and shortest paths
  problems.
\newblock {\em Journal of Mathematical Modelling and Algorithms},
  3(4):349--366, 2004.

\bibitem{SPITZER1970}
F.~Spitzer.
\newblock Interaction of markov processes.
\newblock {\em Advances in Mathematics}, 5(2):246 -- 290, 1970.

\bibitem{Tracy2009}
C.~A. Tracy and H.~Widom.
\newblock Asymptotics in asep with step initial condition.
\newblock {\em Communications in Mathematical Physics}, 290(1):129--154, 2009.

\end{thebibliography}

\end{document}